\documentclass{article}

\usepackage{arxiv}

\usepackage[utf8]{inputenc} 
\usepackage[T1]{fontenc}    
\usepackage{hyperref}       
\usepackage{url}            
\usepackage{booktabs}       
\usepackage{amsfonts}       
\usepackage{amsmath}
\usepackage{bbold}
\usepackage{nicefrac}       
\usepackage{microtype}      
\usepackage{lipsum}		
\usepackage{graphicx}
\usepackage{natbib}
\usepackage{doi}
\usepackage{caption}
\usepackage{subcaption}

\usepackage{booktabs}
\usepackage{multirow}
\usepackage{amsthm}

\newtheorem{theorem}{Theorem}[section]
\newtheorem{lemma}[theorem]{Lemma}

\newcommand{\crps}{\text{CRPS}}
\newcommand{\wcrps}{\text{wCRPS}}
\newcommand{\twcrps}{\text{twCRPS}}

\title{Improving probabilistic forecasts of extreme wind speeds by training statistical post-processing models with weighted scoring rules}

\author{ \href{https://orcid.org/0000-0003-2621-2477}{\includegraphics[scale=0.06]{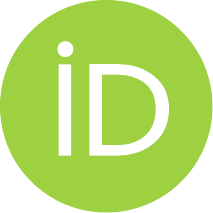}\hspace{1mm}Jakob Benjamin Wessel}\thanks{Correspondence to: j.wessel@exeter.ac.uk. Also at: The Alan Turing Institute, London, United Kingdom} \\
	Department of Mathematics and Statistics\\
	University of Exeter\\
	Exeter, United Kingdom \\
	\And
	\href{https://orcid.org/0000-0002-9830-9270}{\includegraphics[scale=0.06]{orcid.pdf}\hspace{1mm}Christopher A. T. Ferro} \\
	Department of Mathematics and Statistics\\
	University of Exeter\\
	Exeter, United Kingdom \\
    \And
	\href{https://orcid.org/0000-0002-5935-3965}{\includegraphics[scale=0.06]{orcid.pdf}\hspace{1mm}Gavin R. Evans} \\
	Met Office\\
	Exeter, United Kingdom \\
    \And
	\href{https://orcid.org/0000-0003-1421-4010}{\includegraphics[scale=0.06]{orcid.pdf}\hspace{1mm}Frank Kwasniok} \\
	Department of Mathematics and Statistics\\
	University of Exeter\\
	Exeter, United Kingdom \\
}

\renewcommand{\shorttitle}{Improving Probabilistic Forecasts of Extreme Wind Speeds}

\hypersetup{
pdftitle={Improving probabilistic forecasts of extreme wind speeds by training statistical post-processing models with weighted scoring rules},
pdfsubject={stat.AP, stat.ME, physics.ao-ph},
pdfauthor={Jakob Benjamin Wessel, Christopher A. T. Ferro, Gavin R. Evans, Frank Kwasniok},
pdfkeywords={Probabilistic Weather Forecasting, Statistical post-processing, Scoring Rules, Wind speed, Extreme Events},
}

\begin{document}
\maketitle

\begin{abstract}
Accurate forecasts of extreme wind speeds are of high importance for many applications. Such forecasts are usually generated by ensembles of numerical weather prediction (NWP) models, which however can be biased and have errors in dispersion, thus necessitating the application of statistical post-processing techniques. In this work we aim to improve statistical post-processing models for probabilistic predictions of extreme wind speeds. We do this by adjusting the training procedure used to fit ensemble model output statistics (EMOS) models – a commonly applied post-processing technique – and propose estimating parameters using the so-called threshold-weighted continuous ranked probability score (twCRPS), a proper scoring rule that places special emphasis on predictions over a threshold. We show that training using the twCRPS leads to improved extreme event performance of post-processing models for a variety of thresholds. We find a distribution body-tail trade-off where improved performance for probabilistic predictions of extreme events comes with worse performance for predictions of the distribution body. However, we introduce strategies to mitigate this trade-off based on weighted training and linear pooling. Finally, we consider some synthetic experiments to explain the training impact of the twCRPS and derive closed-form expressions of the twCRPS for a number of distributions, giving the first such collection in the literature. The results will enable researchers and practitioners alike to improve the performance of probabilistic forecasting models for extremes and other events of interest.
\end{abstract}

\section{Introduction}

Extreme weather events such as heavy precipitation or windstorms have the potential to cause tremendous damage to lives and livelihoods. To mitigate their impact the development of early warning systems is crucial, which rely on accurate forecasts of relevant extremes. In this work, we aim to improve statistical and machine learning models for probabilistic predictions of extreme wind speeds for applications to the post-processing of numerical weather prediction (NWP) wind speed forecasts. \\

Weather forecasts are usually generated by ensembles of NWP models, each with different initial conditions and possibly perturbed models, to quantify the uncertainty present in atmospheric phenomena. These ensembles, however, often contain biases and errors in dispersion, thus necessitating the application of statistical post-processing techniques to generate accurate and well-calibrated probabilistic forecasts. This can be especially relevant for extreme events, whose forecasts might be improved most from post-processing, but which also require special care when using common techniques \citep{friederichs_forecast_2012, williams_comparison_2014, vannitsem_2018_chapter_5, hess_statistical_2020, velthoen_gradient_2023}.\\

For post-processing of wind speed forecasts with particular focus on extremes, different models have been proposed in the literature. \citet{lerch_comparison_2013} suggest post-processing based on ensemble model output statistics \citep[EMOS, ][]{gneiting_calibrated_2005} -- a commonly applied post-processing technique -- using the truncated normal and generalized extreme value (GEV) distributions as well as a regime switching model. \citet{baran_log-normal_2015} propose post-processing based on log-normal and truncated normal distributions as well as regime switching between the two. \citet{allen_accounting_2021} incorporate weather regimes based on the North Atlantic Oscillation into their wind speed post-processing using truncated logistic distributions, showing that this leads to improved predictive performance at extremes. \citet{oesting_statistical_2017} introduce post-processing of wind gusts forecasts based on a bivariate Brown-Resnick process. \\


Whilst the above authors propose different model structures that enable more accurate inference, we follow a different approach and change the training procedure used to fit models, in order to improve performance at extremes. This can in future work be readily extended to different model specifications. In order to improve the performance of predictive models at extremes, in recent times, a number of different authors have suggested to adjust loss functions used for the training of these models, which are often based on deep neural networks \citep{ding_modeling_2019, chen_novel_2022, hess_deep_2022, scheepens_adapting_2023, olivetti_advances_2024}. However, these works focus mostly on deterministic forecasts, and as we show in section \ref{sec:motivation} the weighting schemes that these authors propose might not be suitable for probabilistic forecasts as they lead to hedging of the forecast distribution. \\

We propose adjusting training by estimating models using weighted proper scoring rules, in particular the so-called threshold-weighted continuous ranked probability score (twCRPS). This proper scoring rule, introduced by \citet{gneiting_comparing_2011} has long been used for evaluation of predictive models, but to the authors' knowledge not yet for systematic parameter optimization. In this work, we focus on post-processing using ensemble model output statistics (EMOS), a typical post-processing method, and introduce the usage of the twCRPS to estimate the model parameters. We show that this improves predictive performance for extremes, for a variety of extremal thresholds. We set out some of the choices and challenges that forecasters need to consider when doing this, and describe the effects of these choices in real and simulated data, closing with recommendations of what to consider when applying this in operational settings. We also give different characterizations of the twCRPS and derive closed-form expressions for a number of predictive distributions.\\

This article is organized as follows. First, in section \ref{sec:data_methods} we will present the data and methods used for post-processing and forecast verification in this study. We will then consider a motivating example in section \ref{sec:motivation} to highlight the need for propriety of scores used for the training of probabilistic forecasting models for extremes. In section \ref{sec:weighted_scoring_rules} we introduce weighted scoring rules and the threshold-weighted continuous ranked probability score (twCRPS) that we use in section \ref{sec:train} to train post-processing models with a focus on extreme wind speeds. We show that this improves forecast performance for extremes and discuss some choices and challenges that forecasters might encounter in sections \ref{sec:first_results}, \ref{sec:threshold} and \ref{sec:datasize}, namely the role of the threshold for training models and the role of the training datasize. In section \ref{sec:combinations} we propose two strategies to balance different objectives that forecasters might encounter when wanting to improve forecasts for extremes: how to balance the forecast performance for the distribution body and the distribution tails. Finally, in section \ref{sec:synth_ex} we present a number of synthetic experiments to better understand the training impact of the twCRPS and interpret our results. We conclude in section \ref{sec:conclusion}, discussing the results and giving some recommendations for forecasters wishing to apply the methodology. In the Appendix we give a number of different characterizations of the twCRPS and derive closed-form expressions for a number of distributions.

\section{Data and methods}
\label{sec:data_methods}
In this section we briefly describe the data and methods used for post-processing of wind speed forecasts and verification.

\subsection{Data}

In this study we consider forecasts of 10m wind speed at surface observation stations in the United Kingdom at a lead time of 48h, initialized between 1 April 2019 and 31 March 2022. The forecasts are taken from the 0000 UTC initialization of the Met Office global ensemble prediction system (MOGREPS-G, \citealt{walters_met_2017, porson_recent_2020}), an 18-member ensemble (including the unperturbed control run) with a horizontal resolution of 20km. The forecasts are bilinearly regridded to the locations of 124 surface synoptic observation (SYNOP) stations in the United Kingdom (see Figure \ref{fig:locations} for the locations) and verified against wind speed observations (10min average of 10m wind speed) there. We summarize the forecast ensemble by the ensemble mean and ensemble standard deviation. We use forecasts issued between 1 April 2019 and 31 December 2020 for training and forecasts issued between 1 January 2021 and 31 March 2022 as an independent test set.

\begin{figure}[ht]
        \centering
        \includegraphics[width=0.35\textwidth]{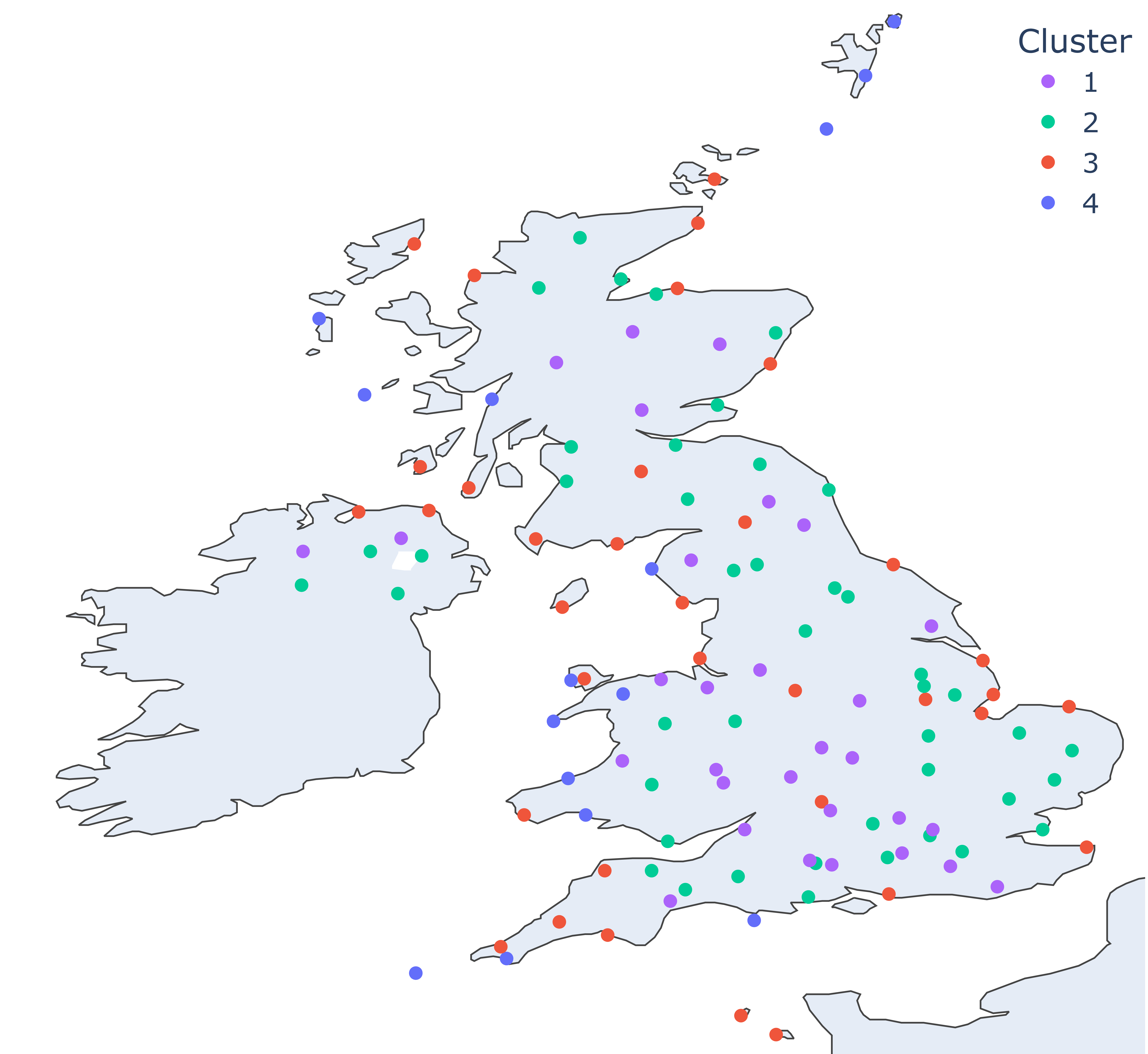}
        \caption{Locations of 124 SYNOP stations used for the verification of wind speed forecasts, together with associated clusters (see section \ref{sec:methods} and table \ref{tab:percentiles} in the appendix for explanation).}
        \label{fig:locations}
\end{figure}

\subsection{Post-processing methods}
\label{sec:methods}

In this study we consider post-processing based on ensemble model output statistics (EMOS) also known as non-homogeneous (Gaussian) regression (NGR, \citealt{jewson_new_2004, gneiting_calibrated_2005}). EMOS is one of the most commonly used post-processing methods, both in research works and operational forecasting environments \citep{hess_statistical_2020, roberts_improver_2023}, owing to its simplicity, robustness and extensibility. Within the EMOS framework the variable of interest (wind speed $W$ here) is assumed to follow a parametric distribution $\mathcal{D}$ whose parameters are modelled in terms of the raw ensemble $x_1, \dots, x_N$, often summarized by ensemble mean $m$ and ensemble standard deviation $s$. In this study, we consider both truncated normal $\mathcal{D} = \mathcal{N}_0$ \citep{thorarinsdottir_probabilistic_2010} and truncated logistic $\mathcal{D} = \mathcal{L}_0$ \citep{messner_extending_2014, scheuerer_probabilistic_2015} as predictive distributions for wind speed (both truncated at zero). In addition to the ensemble mean and ensemble standard deviation we also include sine-cosine transformations of the normalized day of year ($\text{ndoy} := 2 \pi \tfrac{\text{day of year}}{365.25}$) as predictors, to account for seasonality \citep{lang_remember_2020}. We model:
\begin{gather}
            W\, |\, m, s  \sim \mathcal{D}(\mu, \sigma),\\
    \mu = \alpha + \beta m + \lambda_{\mu, {\rm s}} \sin\left(\text{ndoy} \right) + \lambda_{\mu, {\rm c}} \cos \left(\text{ndoy} \right), \\
    \log \sigma = \gamma + \delta s + \lambda_{\sigma, {\rm s}} \sin\left(\text{ndoy} \right) + \lambda_{\sigma, {\rm c}} \cos \left(\text{ndoy}\right).
\end{gather}
Here $\mu$ and $\sigma$ are the location and scale parameters of the predictive distribution $\mathcal{D}$. Sometimes, the predictive standard deviation $s$ is also log-transformed to ensure that $s$ and $\sigma$ are on the same scale, however, this did not make a difference in performance here.\\
EMOS models are usually fit by minimizing either the negative log-likelihood or the continuous ranked probability score (CRPS). For a probabilistic forecast given through a cumulative distribution function (CDF) $F$ and an observation $y$ the latter is defined as:
\begin{align}
\label{CRPS}
    \crps(F, y) = \int_{-\infty}^{\infty} \left[F(x) - \mathbb{1}\{x \geq y\}\right]^2 dx.
\end{align}
Both the CRPS and the negative log-likelihood are proper scoring rules \citep{gneiting_strictly_2007}, meaning that in expectation they are minimal when the forecast distribution is the true distribution (see section \ref{sec:weighted_scoring_rules} for some more discussion of propriety). Closed-form expressions of the CRPS exist for the truncated normal and truncated logistic distributions \citep[see][]{jordan_evaluating_2019} and are used for the minimization. \\

EMOS models can be fit individually to a single site, or to sets of sites. For example, the configuration of the Met Office IMPROVER system described in \citet{roberts_improver_2023} fits a single EMOS model for wind speed to post-process forecasts jointly to measurement sites, including the 124 stations presented here. As we will highlight in section \ref{sec:datasize} large data is needed for training using weighted scoring rules. Thus, to increase the datasize for training and evaluation of each EMOS fit we use a semilocal post-processing approach as introduced by \citet{lerch_similarity-based_2016}. In this approach EMOS models are fit to clusters of similar stations. We cluster based on the observed climatology of wind speed at each station, which has the advantage that quantile-based thresholds are homogeneous within each cluster. \\

Following \citet{lerch_similarity-based_2016} we compute K = 24 features for each station, corresponding to the $1/(K+1), \dots K/(K+1)$ quantiles of the location-wise distribution of observed wind speed in the training dataset. In a second step k-means clustering \citep{murphy_machine_2012} is used to identify sets of similar stations from the features of all stations. We use the ``elbow-method" to identify four distinct clusters, shown in Figure \ref{fig:locations}. The clusters broadly correspond to coastal locations (red), more exposed coastal locations (blue), more exposed inland locations (green) and less exposed inland locations (purple). In a second step EMOS models are then fit to jointly post-process wind speed forecasts for all stations within one cluster. Although four clusters are used here, the results below were consistent for other choices of cluster numbers, as long as a minimum amount of stations were present in each cluster. The 80th and 90th percentile wind speed in each cluster are shown in Table \ref{tab:percentiles} in the appendix.\\

Unless otherwise stated in this paper we use a Broyden-Fletcher-Goldfarb-Shanno (BFGS) quasi-Newton optimizer \citep{nocedal_numerical_2006} with empirical gradient estimates for optimizing scoring functions, which falls back onto a Nelder-Mead optimizer \citep{nelder_simplex_1965} should it fail. Using analytic gradients is also possible and might increase the optimization efficiency, but we leave such exploration for future work. We follow the practice of the \texttt{crch} package \citep{messner_crch_2022} and use the coefficients of a linear regression as starting values for the location parameters ($\alpha, \beta, \lambda_{\mu, {\rm s}}, \lambda_{\mu, {\rm c}}$) and the log-transformed standard error as the starting value of the intercept of the scale parameter ($\gamma$), whilst setting all other scale parameters ($\delta, \lambda_{\sigma, {\rm s}}, \lambda_{\sigma, {\rm c}}$) initially to zero. \\

\subsection{Forecast verification}

The goal of probabilistic forecasting has been put by \citet{gneiting_calibrated_2005} as to ``maximize \textit{sharpness} subject to \textit{calibration}". Sharpness is a measure of the forecast distribution only and refers to the concentration of the forecast density function, whilst calibration is a joint property between observations and forecast and corresponds to the consistency between the forecast distributions and the observations. As this paper focuses on forecasts of extreme events, this will also guide our forecast verification. \\

Sharpness of forecasts can be evaluated via statistics of the predictive distribution. More often however, this is evaluated jointly with calibration via the usage of proper scoring rules. In particular we will use the CRPS and threshold-weighted CRPS for forecast evaluation, which are both proper scoring rules, and which we will introduce in section \ref{sec:weighted_scoring_rules}.  Calibration on the other hand, which is a ``necessary condition for the optimal use and value of a forecast" \citep{vannitsem_2018_chapter_6}, is usually also assessed individually. Most commonly calibration is evaluated using probability integral transform (PIT) histograms \cite[see for example][]{vannitsem_2018_chapter_6}. These are histograms of the predictive CDF $F_i$, evaluated at the observation $y_i$:  $F_i (y_i)$. Under probabilistic calibration PIT values should be uniform, corresponding to a flat histogram. Alternatively, they can also be visualized against theoretical quantiles from a uniform distribution in a quantile-quantile plot (QQ-plot), in which case the PIT values should lie on the diagonal axis. \\

As pointed out by \citet{allen_weighted_2023, gneiting_regression_2023, allen_tail_2024}, probabilistic calibration as assessed by uniformity of PIT histograms does not necessarily imply probabilistic calibration for all subsets of the observations that are of interest. For example, a forecast could be probabilistically calibrated overall, without being probabilistically calibrated in the tail of the observational distribution. To assess probabilistic calibration for tail events one can evaluate conditional PIT histograms as suggested by \citet{allen_weighted_2023, mitchell_censored_2023}. For a forecast $F_i$ define the excess distribution over a threshold $\tau$ as:
$$F_{i, \tau} (x) = \frac{F_i(x) - F_i(\tau)}{1-F_i(\tau)},\quad x \geq \tau$$
The conditional PIT (CPIT) values are the values of the forecast excess distribution evaluated at observations $y_i \geq \tau$: $F_{i, \tau}(y_i)$. These should again be uniform if the forecast is probabilistically calibrated over the threshold (in the tails) which can again be assesed using histograms or QQ-plots. \\

The conditional PIT assesses the probabilistic calibration for the distribution of threshold exceedances. \citet{allen_tail_2024} also introduce the notion of probabilistic tail calibration, which in addition to the excess distribution also incorporates the probability of threshold exceedance. A forecast is tail calibrated if the following holds for $\tau$ approaching the upper endpoint of the distribution of observations e.g. $\tau \to \infty$ :
$$
\frac{\mathbb{P}(F_{\tau}(Y) \le u, Y > \tau)}{\mathbb{E}(1 - F(\tau))} = u \quad \text{for all} \quad u \in [0, 1].
$$
This can be empirically assessed using the following ratio:
$$
\hat{R}_{\tau}(u) = \frac{\sum_{i \in I_{\tau}} \mathbb{1}\{z_{i}^{\tau} \leq u \}}{ \sum_{i = 1}^{n} (1 - F_{i}(\tau))} = \frac{1}{|I_\tau|} \sum_{i \in I_{\tau}} \mathbb{1}\{z_{i}^{\tau} \leq u \} \cdot \frac{|I_\tau|}{ \sum_{i = 1}^{n} (1 - F_{i}(\tau))}, \quad u \in [0, 1].
$$
here $z_1^\tau, \dots z_n^\tau$ are the conditional PIT values and $I_{\tau} = \{i: y_i > \tau\}$. This ratio should be equal to or close to $u$ for all values of $u\in [0,1]$. It corresponds to the empirical distribution of conditional PIT values, multiplied by the occurrence ratio of realized over expected threshold exceedances. To summarize the behavior of $\hat R_\tau(u)$ the authors also suggest using the following as a measure of tail miscalibration (TMCB):
\begin{align}
    \label{eq:TMCB}
    \text{TMCB} = \sup_{u \in [0, 1]} \left| \hat{R}_{\tau}(u) - u \right|.
\end{align}

\section{Motivating the use of weighted proper scoring rules}
\label{sec:motivation}

In order to improve the performance of data-driven models for prediction of extreme events multiple authors propose adjusting loss functions used to learn parameters of these models \citep[e.g.][]{ding_modeling_2019, chen_novel_2022, olivetti_advances_2024}. Among others \citet{hess_deep_2022} and \citet{scheepens_adapting_2023} propose weighting schemes for loss functions to weight misprediction at extremes stronger during training. Many of the above authors focus on deterministic predictions of extremes, however as \citet{friederichs_forecast_2012} point out to allow for the assessment of the potentially large uncertainty in predictions of extreme events, such predictions should be ``probabilistic in nature". In this section we will study the effect of a loss function weighting strategy for probabilistic extreme event forecasting in the setting of post-processing. This is to motivate the need for proper scoring rules, when training probabilistic forecasting models, which we discuss in the next section. \\

\citet{hess_deep_2022} propose a weighting scheme for loss functions to penalize stronger the mis-prediction of high precipitation values during training. They only consider deterministic forecasts, trained using root mean squared error (RMSE), however, we consider this weighting scheme adapted for the training of probabilistic post-processing models for wind speed. We consider EMOS post-processing models using a truncated normal distribution, trained  by minimizing the following loss (henceforth referred to as HB-loss):
\begin{align}
\label{L1}
     L (F, y) := w_{a, b} (y) \ \crps(F, y)\ \ \ \ \text{ with } \ \ \ \ w_{a, b}(y) := \min \left[1, \exp \left(\tfrac{y - a}{b}\right) \right].
\end{align}
This weighting scheme means that prediction quality (measured through the CRPS) of high observed wind speed values will carry more weight during training. We choose $a$ as the cluster-wise 90th percentile of the observations of the training dataset and $b$ as the cluster-wise standard deviation of the observational training dataset. We score predictions against EMOS models trained classically using the CRPS.\\

Investigating the trained model's performance on the test set we find that when training EMOS post-processing models using the HB-loss the resulting models have much lower HB-loss on the test set (9.5\% lower) compared to the CRPS-trained model. Looking at the dataset of test set extremes -- wind speed values over the location-wise 90th percentile of the training set $\tau$ -- we find that the HB-trained model is far better than the CRPS-trained model having 86.45\% better RMSE and 26.3\% better CRPS respectively on this dataset. The model has good out-of-sample HB-loss performance and seems to be able to capture the intensity of extremes quite well. \\

Next to the intensity we also analyse the performance of the models at predicting the occurrence of extremes by evaluating the Brier score -- a proper scoring rule for binary events/threshold exceedances:
\begin{align}
    BS_\tau(F, y) := \left[\left( 1- F(\tau)\right) - \mathbb{1}\{y \geq \tau\} \right]^2.
\end{align}
We find that the HB-trained model strongly overpredicts extremes, having 18\% worse Brier score than the CRPS-trained model. A similar picture can be seen by considering PIT histograms. Figure \ref{fig:pit_hb} shows the histograms for the CRPS and HB-trained model. Whilst the CRPS-trained model seems approximately calibrated the HB-trained model has a right tail that is much too heavy and seems to strongly underpredict values in the body of the climatological distribution. \\

\begin{figure}[ht]
        \centering
        \includegraphics[width=0.6\textwidth]{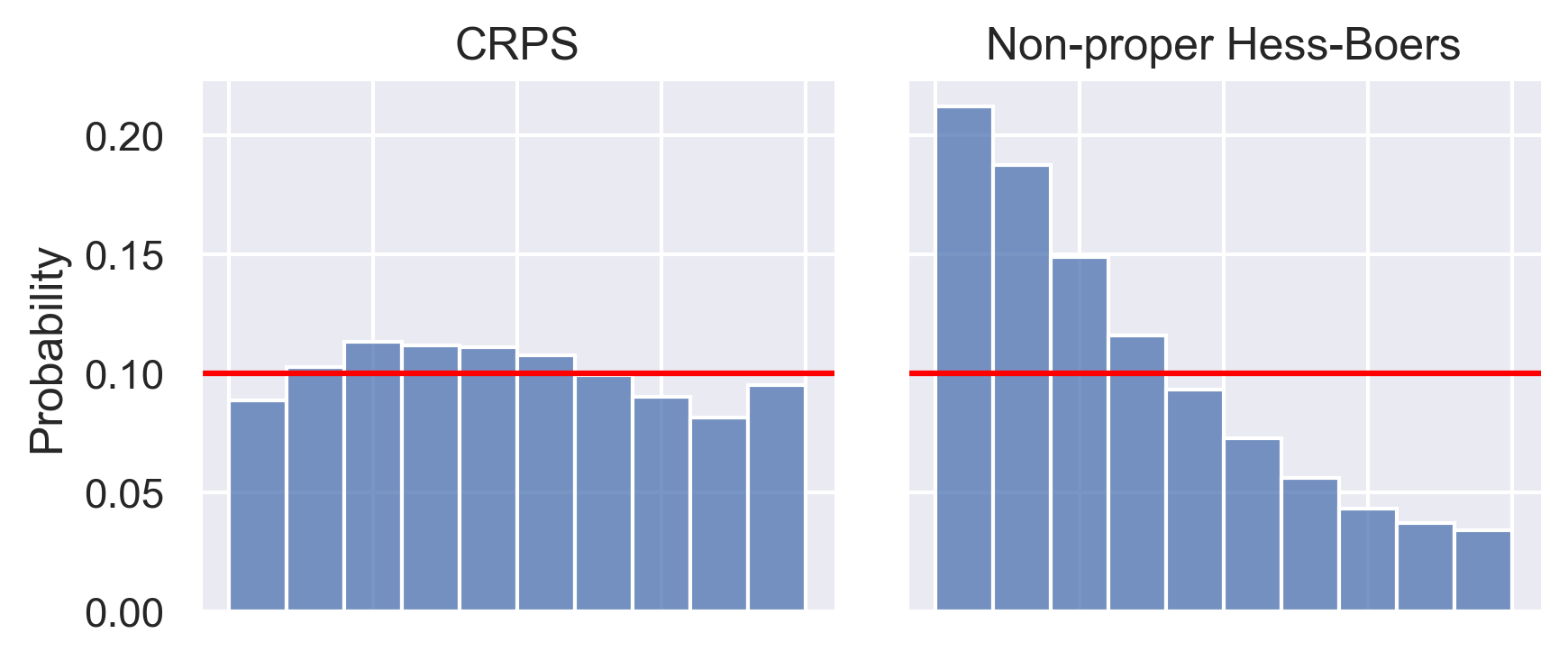}
        \caption{PIT histograms for wind speed post-processing models using truncated normal distributions, trained using CRPS loss (left) and HB-loss (equation \ref{L1}, right). Pooled over all clusters.}
        \label{fig:pit_hb}
\end{figure}

When evaluating only on test set extremes the HB model seems to be a suitable prediction model, however, when considering the full dataset the model strongly overpredicts the occurrence of extremes and seems strongly mis-calibrated.  A similar result also occurs with a very related strategy of oversampling extremes prior to training, often used to adjust for class imbalance in classification problems \citep{van_den_goorbergh_harm_2022}. The need to evaluate not only the intensity of extreme event forecasts, but to do this jointly with the predicted occurrence has been highlighted by many authors and is commonly referred to as the ``forecaster's dilemma" \citep{lerch_forecasters_2017}. If only forecasts of intensity are evaluated, then this encourages forecasters to overpredict extreme values, something that is not a problem in cases where there is no penalty for false-positive extreme event warnings. However, in real applications this will quickly erode trust in any forecasting system and lead to economic costs. The CRPS is a proper scoring rule, thus using it for training and evaluation encourages honesty of the forecasters and calibrated forecasts. However, when weighting the CRPS by an observed outcome, for example by restricting evaluation to the dataset of test set extremes or by using a weighting such as the HB-weighting, propriety is lost. This makes the score unsuitable for objective evaluation and for training (see next section).

\section{Weighted scoring rules for assessing forecasts of extremes}
\label{sec:weighted_scoring_rules}

 Proper scoring rules \citep{gneiting_strictly_2007} have been introduced for the evaluation (and training) of probabilistic forecasts. They assess both calibration and sharpness. Briefly, a scoring rule $S$ is proper when the following holds for any probabilistic forecast $F$ and true distribution $G$:
\begin{align}
\label{proper}
    \mathbb{E}_{y \sim G} S(F, y) \geq \mathbb{E}_{y \sim G} S(G, y).
\end{align}
Here $S(F, y)$ is the score of the predictive CDF at the observation $y$ and $\mathbb{E}_{y \sim G} S(F, y)$ denotes its  expectation over the distribution of $G$. A scoring rule is thus called proper when in expectation it is minimal at the true distribution. It is called strictly proper if equality in equation \ref{proper} holds only when $F = G$. \\

When wanting to focus probabilistic forecast evaluation on certain events of interest such as extreme events over a threshold, one can try to restrict the evaluation to these events or weight them stronger. Mathematically, both correspond to constructing a scoring rule weighted by an observation: $\tilde S (F,y) := w(y) \ S(F, y)$. However, as \citet{gneiting_comparing_2011} and \citet{lerch_forecasters_2017} show, if the weighting function is non-constant the resulting scoring rule becomes non-proper. Concretely, the expected weighted score $\mathbb{E}_{y \sim G} \tilde S (F, y)$ is no longer minimal at $F = G$ but at the predictive distribution corresponding to the density
\begin{align}
    f(y) = \frac{w(y)g(y)}{\int w(x) g(x) dx} \propto w(y)g(y),
\end{align}
where $g$ is the density corresponding to $G$. Thus, when evaluating probabilistic forecasting models using scoring rules restricted to an outcome (such as events over a threshold) or scoring rules weighted by some weighting function (such as the HB-loss) this encourages hedging of the forecast distribution, meaning a forecaster is encouraged to deviate from the true distribution in their prediction by an amount determined by the weighting function. This also affects model training as can be seen in section \ref{sec:motivation}.\\

To address this problem \citet{gneiting_comparing_2011} introduce proper weighted scores and especially the weighted continuous ranked probability score (wCRPS). These remain proper whilst focusing inference onto certain values of interest. The weighted CRPS is defined as
\begin{align}
    \wcrps(F,y) =  \int_{-\infty}^{\infty} w(x) \left[F(x) - \mathbb{1}\{x \geq y\}\right]^2 dx,
\end{align}
for a nonnegative weighting function $w$. This reduces to the unweighted CRPS (equation \ref{CRPS}) when $w(x) \equiv 1$. The canonical weighting is indicator weighting: $w_\tau(x) = \mathbb{1}\{x \geq \tau\}$, when there is interest in values above (or below) a certain threshold $\tau$, which then leads to the so-called threshold-weighted CRPS (twCRPS):
\begin{align}
    \label{eq:twcrps}
    \twcrps_\tau(F,y) =  \int_{\tau}^{\infty} \left[F(x) - \mathbb{1}\{x \geq y\}\right]^2 dx.
\end{align}
The weighted CRPS and especially the threshold-weighted versions have become frequently used tools for forecast evaluation of extreme events \citep[e.g.][]{lerch_comparison_2013, baran_log-normal_2015, baran_censored_2016, allen_incorporating_2021}. They have also been generalized to multivariate versions \citep{allen_weighted_2023, allen_evaluating_2023}. In the following we will focus on the twCRPS as given in equation \ref{eq:twcrps}.\\

The CRPS corresponds to the integral over the Brier score over all possible thresholds, whilst the twCRPS with threshold $\tau$ restricts the evaluation of this integral only to thresholds of interest, over $\tau$. Evaluation of a predictive CDF $F$ using the twCRPS with threshold $\tau$ corresponds to evaluating the censored CDF $\tilde F_\tau$, censored at $\tau$ using the classical CRPS, formally:
\begin{lemma}
\label{lemma_sam_allen}
    The threshold-weighted CRPS at threshold $\tau$ of the distribution $F$ is equal to the CRPS of the censored distribution $\tilde F_\tau$ with censored observation $\tilde y$:
    \begin{align}
        \twcrps_\tau (F, y) = \crps(\tilde F_\tau, \tilde y)
\end{align}
where $\tilde y := v_\tau(y) = \max(y, \tau)$ and $\tilde F_\tau$ is the CDF censored at $\tau$:
    \begin{align}
    \tilde F_\tau(x) = \begin{cases}
        0 &\text{ if } x < \tau \\
        F(x) &\text{ if } x \geq \tau
    \end{cases}.
\end{align}
\end{lemma}
This is alluded to already in \citet{allen_weighted_2023} and implemented in the scoringRules package \citep{allen_weighted_2024}. The proof is given in the Appendix. Thus intuitively when using the twCRPS a forecaster effectively uses the full CRPS to evaluate the distribution $\tilde{F}_\tau$ that has all probability mass below the threshold collapsed onto the threshold itself, whilst remaining the same above the threshold. That way the exact form of the distribution below $\tau$ only influences the score via the probability mass for exceeding or falling below the threshold.\\

Closed-form expressions of the twCRPS exist for a number of distributions, for example \citet{allen_incorporating_2021} derive the twCRPS for the truncated logistic and \citet{wessel_lead-time-continuous_2024} the twCRPS for the truncated normal distribution. To facilitate application of the twCRPS in future work the authors have collected from the literature and manually derived a number of closed-form twCRPS expressions for many often-used predictive distributions: the normal, logistic, Laplace, Student's t, gamma, log-normal, log-logistic, generalized Pareto, exponential and uniform distributions as well as censored and truncated versions of these. These expressions, together with a number of different characterizations of the twCRPS, can be found in the Appendix.\\

The twCRPS also admits a formulation as an expectation \citep{taillardat_evaluating_2023, allen_weighted_2023}:
\begin{align}
    \label{twcrps_expectation}
    \twcrps_\tau (F, y) = \mathbb{E}_{X \sim F} |v_\tau(X)-v_\tau(y)| - \frac{1}{2} \mathbb{E}_{X, X' \sim F} |v_\tau(X)-v_\tau(X')|,
\end{align}
where $v_\tau = \max(y, \tau)$ as above. The formulation of the twCRPS as an expectation can be used to approximate it via samples from a predictive distribution, i.e. using the estimator
\begin{align}
    \label{eq:sample_twcrps}
    \widehat{\twcrps_\tau} (F, y) = \frac{1}{N} \sum_{i = 1}^N |v_\tau(X_i)-v_\tau(y)| - \frac{1}{2 N^2} \sum_{i, j = 1}^N |v_\tau(X_i)-v_\tau(X_j)|,
\end{align}
where $X_1, \dots, X_N \sim F$ are samples from the predictive CDF. This can be useful for example in cases where no closed-form expression is available.\\

Although the twCRPS has been frequently used for evaluation of probabilistic forecasting models for extremes, it has to this date not been used for training (post-processing) models in the literature or operations. However, given that weighted and threshold-weighted scores are proper -- something that section \ref{sec:motivation} showed is important for model training -- and the relevance of reliable forecasts for extremes, we will study their usage for the training of EMOS post-processing models for wind speed, for both truncated normal and truncated logistic EMOS models. \\

When comparing different models we will frequently use skill scores in the next section. Given forecasts $F_1, \dots, F_N$, reference forecasts $F^{\text{ref}}_1, \dots F^{\text{ref}}_N$ and observations $y_1, \dots, y_N$ then a skill score in percent is given as:
\begin{align*}
    \text{Skill } (\%) = 100\cdot\left(1- \frac{\sum_{i=1}^N \wcrps(F_i,y_i)}{\sum_{i=1}^N\wcrps(F^{\text{ref}}_i, y_i)}\right),
\end{align*}
for any CRPS weighting function $w$. It thus corresponds to the percentage improvement in the score over the reference forecasts.\\

\section{Results: training with weighted scoring rules}

\label{sec:train}
We study the usage of the weighted CRPS for the training of post-processing models, to improve probabilistic predictions of extreme wind speeds. The weighted CRPS is a proper scoring rule, thus using it for evaluation and training will favour calibrated-sharp forecasts. We focus on the canonical threshold weighting due to the availability of closed-form expressions, which helps with the optimization. We consider training EMOS models using the following proper scores:
\begin{itemize}
    \item CRPS,
    \item Negative log likelihood,
    \item twCRPS with a threshold $\tau$.
\end{itemize}
Here the CRPS and log-likelihood models serve as baselines, whilst the twCRPS puts special emphasis on values above a threshold $\tau$. For optimizing the score functions we use the closed-form expressions for the CRPS and twCRPS for the truncated normal and truncated logistic distributions (see Appendix \ref{sec:closed_form}). We also tried optimization using the sample twCRPS estimator (equation \ref{eq:sample_twcrps}) and samples from the predictive distributions, but this proved slower and much less robust than optimization using the closed-form expressions.\\

\subsection{Effect of weighted training}
\label{sec:first_results}

We consider thresholds $\tau$ defined as the location-wise 80th and 90th percentile of the training dataset and train using the scores in section \ref{sec:train}. In practical applications most likely absolute thresholds defined by user needs will be of relevance, however, to show the general applicability of the method we use percentile based thresholds here as in \citet{sharpe_how_2018, sharpe_verification_2019, sharpe_colourful_2022}. Table \ref{tab:percentiles} in the Appendix shows the 80th and 90th percentile threshold in each cluster of stations. We report skill scores relative to the model trained using CRPS (``CRPS-trained model").\\

\begin{figure}[ht]
    \centering
     \includegraphics[width=\textwidth]{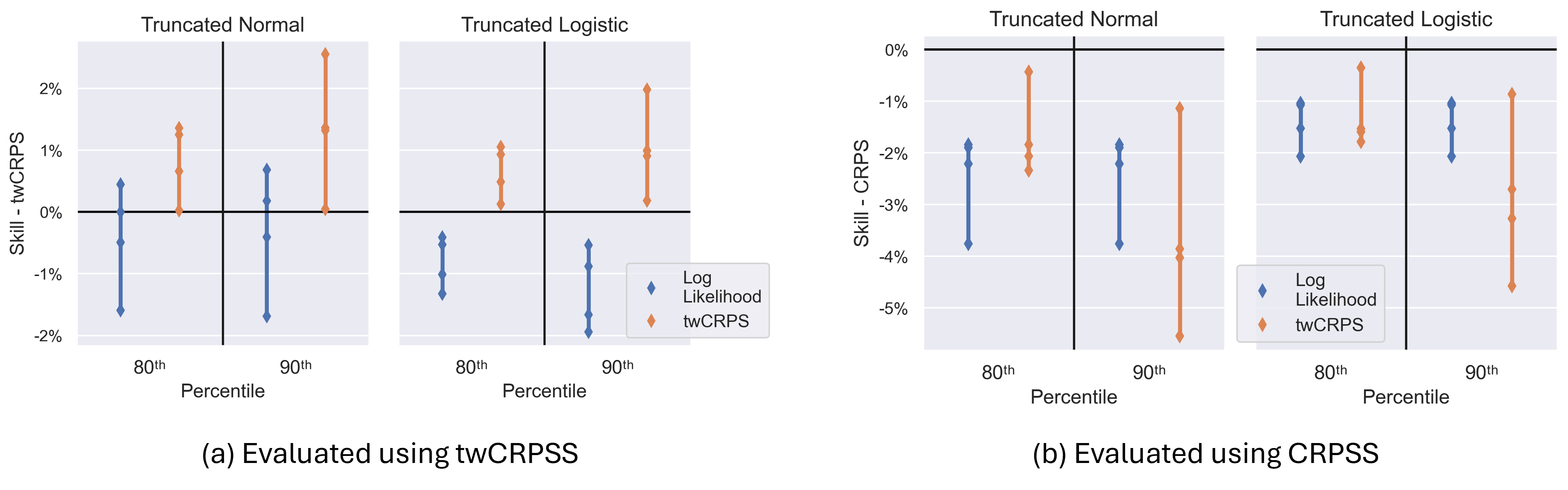}
    \caption{Threshold-weighted CRPS skill score (left, a) and CRPS skill score (right, b), both relative to the CRPS-trained model, for wind speed EMOS using truncated normal and truncated logistic distributions. Models are trained using log likelihood, and twCRPS. Thresholds for training and evaluation of the twCRPS are defined as locationwise 80th and 90th percentile of the training dataset. The diamonds correspond to the results within each cluster and the bands to the spread.}
    \label{fig:twcrps_crps}
\end{figure}

Figure \ref{fig:twcrps_crps}a shows for all models the threshold-weighted CRPS skill score over the CRPS-trained model evaluated on the test set, for the 80th and 90th locationwise percentile thresholds (used for both training and evaluation) for both truncated normal and truncated logistic predictive distributions. Training using a twCRPS leads to improved predictions of extremes, as measured by the twCRPS on the test set. These improvements average around $0.8$\% for the 80th and around $1.3$\% for the 90th percentile threshold and up to $\sim 1.4$\% and $\sim 2.5$\% respectively. The twCRPS improvements also show in terms of Brier skill score (see Table \ref{bss} in the Appendix for the Brier skill score for both distributions for a variety of evaluation thresholds). The strongest improvements in terms of twCRPS are in cluster 1 corresponding to inland locations, followed by clusters 3 and 2 corresponding to coastal and more exposed inland locations, with cluster 4 (more exposed coastal locations) last. The log-likelihood training performs worse in terms of CRPS and twCRPS than the CRPS-trained model which is in line with previous studies \citep{gneiting_calibrated_2005, gebetsberger_estimation_2018}.\\

The improvements in terms of twCRPS however seem to be linked to a loss in CRPS as can be seen in Figure \ref{fig:twcrps_crps}b which shows the CRPS skill over the CRPS-trained model, evaluated on the test set. The twCRPS-trained model loses CRPS skill, compared to the CRPS-trained model for both the 80th and even stronger for the 90th percentile threshold. Performance at the latter threshold improves more in terms of twCRPS, but also decreases more in terms of CRPS, irrespective of the predictive distribution. There seems to be a trade-off between CRPS and twCRPS performance. We investigate this more closely in Figure \ref{fig:trade_off}, which shows the twCRPS skill improvement against the CRPS skill loss for all clusters and both truncated normal and truncated logistic, for 80th and 90th percentile thresholds. There seems to be a near-linear association between CRPS skill loss and twCRPS skill improvement. Larger twCRPS improvements also seem to lead to larger CRPS losses, which are both associated with the larger of the two thresholds. In section \ref{sec:combinations} we will study ways of managing the trade-off  between CRPS and twCRPS performance. \\

\begin{figure}[ht]
    \centering
     \includegraphics[width=0.6\textwidth]{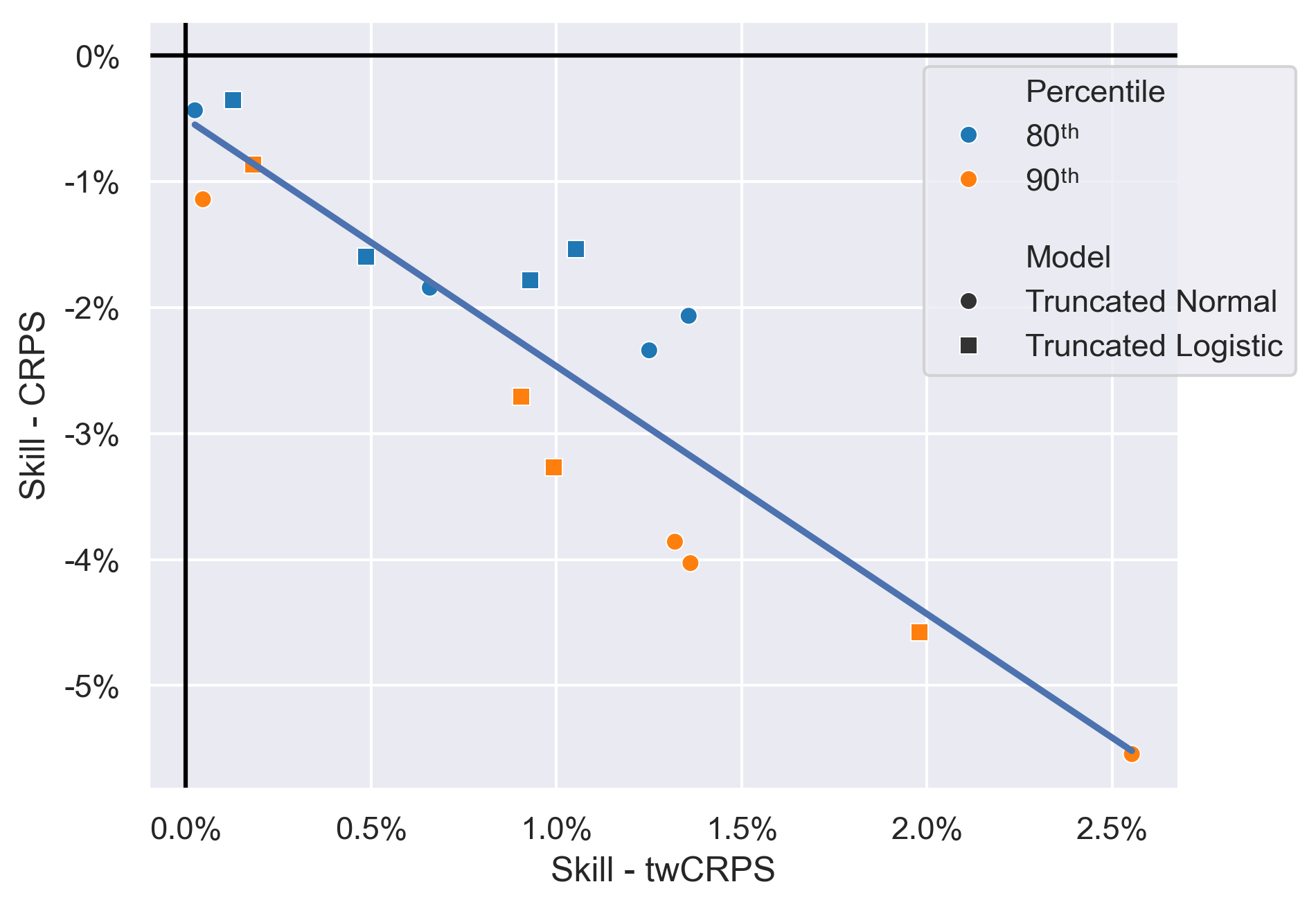}
    \caption{CRPS skill loss against twCRPS skill gain,  both relative to the CRPS-trained model, for all clusters, truncated normal (points) and truncated logistic (squares) predictive distribution and 80th (blue) and 90th (orange) percentile training and evaluation thresholds. The line corresponds to the line of best fit.}
    \label{fig:trade_off}
\end{figure}

We also investigate the effect of twCRPS based training onto extremal calibration, focusing on the truncated normal distribution and 90th percentile threshold, although results for the other models look similar. Figure \ref{fig:extr_calibration}a shows QQ-plots of the conditional PIT values against quantiles from a uniform distribution for CRPS-, log-likelihood and twCRPS-based training. Under probabilistic calibration of the threshold exceedance distribution, these quantiles should be close to the diagonal. twCRPS-based training seems to improve extremal calibration, including in cluster 4 which has the most extreme wind speeds but in which only a small twCRPS score improvement can be found. A similar picture appears when considering the tail miscalibration TMCB (equation \ref{eq:TMCB}). This is shown in Figure \ref{fig:extr_calibration}b as a function of the evaluation threshold. For evaluation thresholds close to the 90th percentile training threshold the twCRPS-trained model performs similarly as the CRPS and log-likelihood trained one. However for higher thresholds it shows much better tail calibration. This is similar also for the 80th percentile model, which shows substantially improved calibration over the 90th percentile (not shown). The improvements in calibration thus seem to hold for a number of thresholds.\\

In summary, by optimizing the twCRPS one can improve extreme event performance of forecasts, as measured by twCRPS, Brier Score and extremal calibration. This however comes at some cost in terms of CRPS -- measuring the forecast performance for the whole distribution and the distribution body. The trade-off between CRPS and twCRPS performance can be seen as a classic distribution body-tail trade-off where an improved modelling of a distribution's tails might come at the cost of capturing the distribution's body and vice-versa. This seems to hold for training thresholds defined as 80th and 90th percentile of the distribution, however, in the next section we will study the effect of varying the training threshold.

\begin{figure}[ht]
    \centering
     \includegraphics[width=\textwidth]{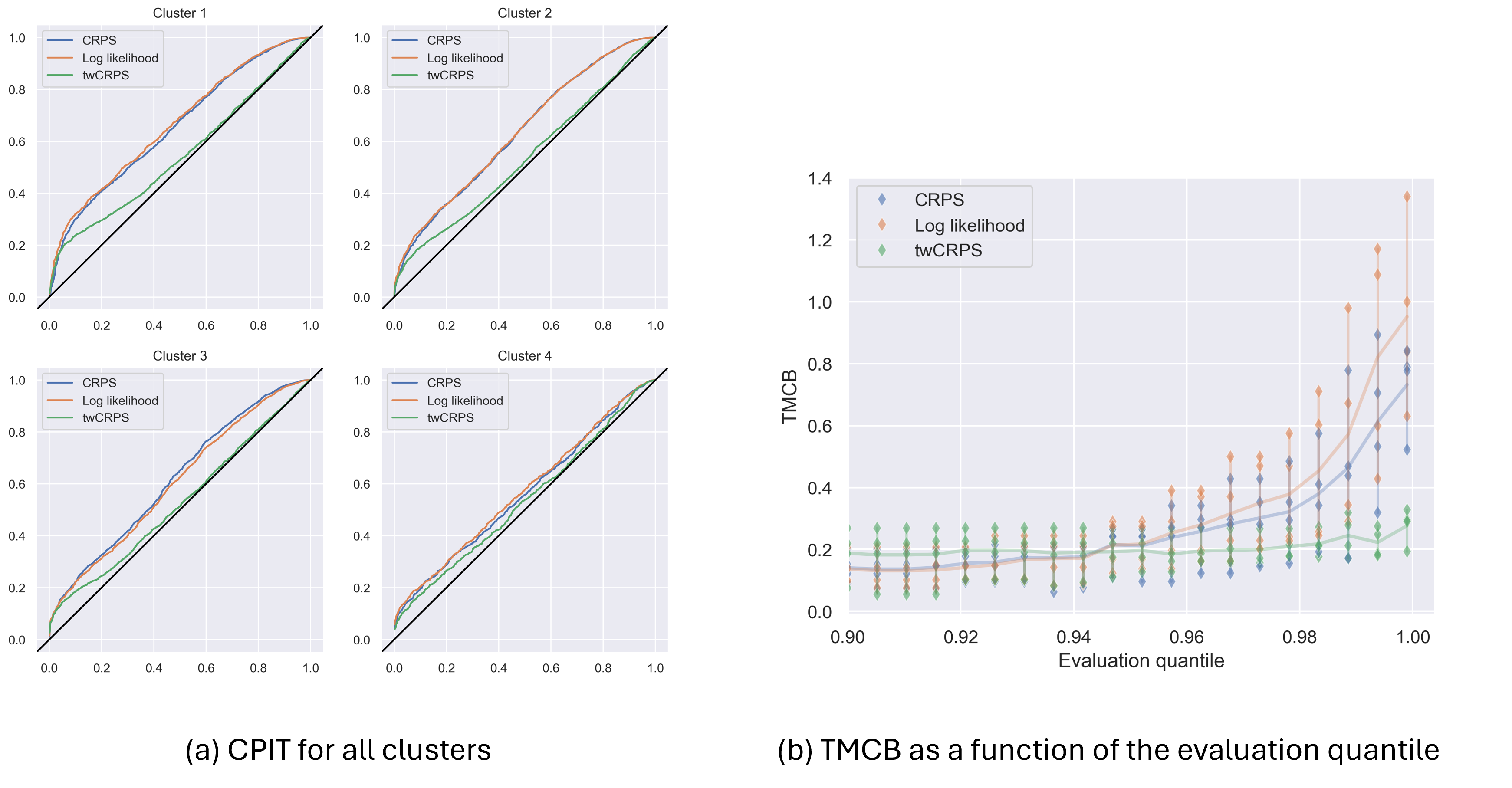}
    \caption{Extremal calibration: a) By cluster: QQ plots of quantiles from a uniform distribution against of conditional PIT values for CRPS-, log-likelihood- and twCRPS-trained models, 90th percentile, truncated normal predictive distribution. b) TMCB as a function of the evaluation threshold for the CRPS-, log-likelihood and twCRPS-trained models, 90th percentile model, truncated normal predictive distribution. The diamonds represent the results in each cluster, whilst the line corresponds to the mean. }
    \label{fig:extr_calibration}
\end{figure}

\subsection{Effect of the training threshold}
\label{sec:threshold}
We study the effect of varying the threshold for training using the twCRPS and analyse this for the truncated normal case. Results for the truncated logistic distribution are similar. We train post-processing models using the twCRPS with thresholds as 60th, 80th, 90th and 95th percentile and evaluate against the twCRPS with a range of different thresholds between the 0th (corresponding to the unweighted CRPS) to 99th percentile of training dataset. Figure \ref{fig:varying_threshold} shows by cluster the twCRPS skill (relative to the CRPS-trained model) as a function of the evaluation threshold (in m/s) for all four trained models. Results in terms of quantile can be found in Figure \ref{fig:varying_threshold_relative}, and for the 90th percentile model a comparison against the raw ensemble can be seen in Figure \ref{fig:comp_raw_ensemble}, both in the Appendix. The model trained for a certain threshold generally seems to perform well at this threshold and at thresholds slightly lower and higher than the one it was trained for. All models experience a drop in performance near the 99th percentile, which generally corresponds to very high wind speed values which naturally will be very difficult to predict. Training with a twCRPS for a high quantile threshold also leads to skill loss at lower quantiles and it seems the higher the training threshold / quantile the worse the twCRPS skill loss for lower evaluation thresholds / quantiles also becomes, manifesting in increasing slope of the skill curves. This means that model performance can be improved for many thresholds by training with a twCRPS, but should the distribution body performance be important it might be sensible to use lower thresholds. In these cases however the skill improvements for high quantiles are also more modest. Some improvement can also be seen in terms of Brier score (see Figure \ref{fig:brier_score_training_quantiles} in the Appendix), at least for cluster 1, 2, and 3, although results are more varied.

\begin{figure}[h!]
        \centering
        \includegraphics[width=0.6\textwidth]{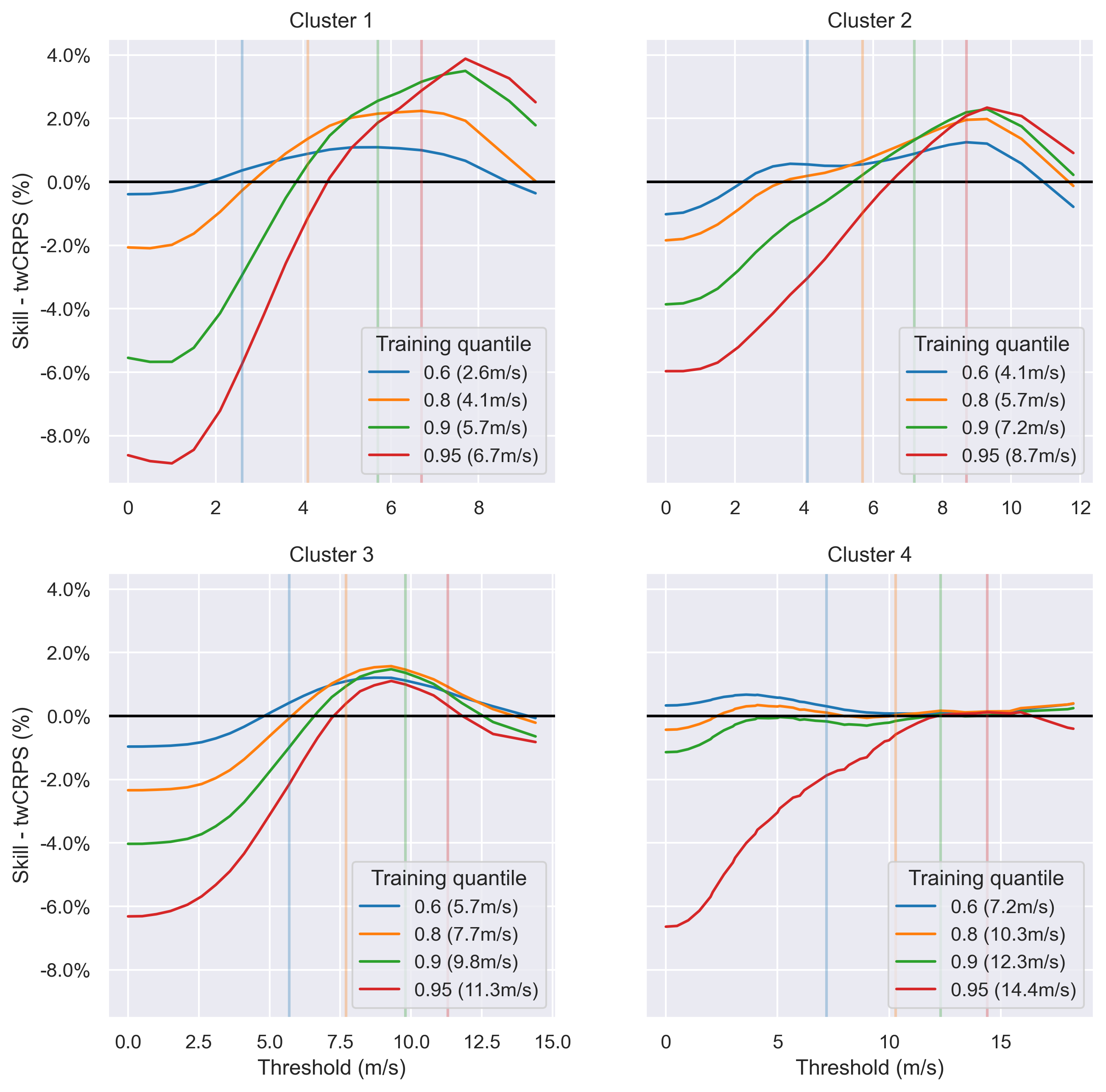}
        \caption{twCRPS skill score (over the CRPS-trained model) by cluster as a function of the evaluation threshold for twCRPS-trained models trained using different thresholds. Thresholds on the x-axis correspond to wind speeds between the 0th and 99th percentile of the training distribution. The vertical lines indicate the thresholds for which models were trained.}
        \label{fig:varying_threshold}
\end{figure}

\subsection{Effect of training datasize}
\label{sec:datasize}

We study the effect of the training datasize for training EMOS models using the twCRPS. For this we vary the number of stations that twCRPS-based EMOS models are trained on, allowing us to probe the amount of data needed. We define subsets of N stations by sampling at random with replacement from the dataset of 124 stations, ignoring the previous cluster assignments. EMOS models are then trained on the dataset of stations within this subset/cluster using CRPS and twCRPS, in order to calculate the twCRPS skill of the twCRPS-model on the subset-specific test set. This process is bootstrapped 50 times to get an indication of the uncertainty. We focus on the truncated normal predictive distribution for 80th and 90th percentile models and for models trained on the whole training data (one year and 9 months), as well as models trained on only one year of training data. Figure \ref{fig:training_data_dep} visualizes the skill score as a function of the number of locations included in the training set.\\

The training datasize seems to have substantial influence on whether twCRPS-based training is successful. Consistent positive twCRPS skill on the test set can only be found when enough locations are included in each subset/cluster, for both percentiles and both the full and half training data. More locations are needed for the 90th percentile model than for the 80th percentile model and for the models using half the training data compared to the models trained on the full training data. Little benefit in terms of twCRPS skill can be seen for training on individual stations in this dataset, however this might be different if more training data per station would be available. Interestingly, even if there is no improvement in terms of twCRPS skill for smaller clusters (including locationwise training), there is generally still an improvement in terms of extremal calibration (not shown).\\

This dependence on training datasize might be due to multiple reasons. Extensive training data might be needed to infer tail properties well. Also, \citet{lerch_forecasters_2017} report low discriminative power of the twCRPS in Diebold-Mariano test. Thus, in order for having adequate power to distinguish well between candidate EMOS models, large amounts of data might be needed. The training data requirements also seem to vary with the thresholds, with successful training of twCRPS models for larger thresholds needing more training data than for lower thresholds. 

\begin{figure}[h!]
        \centering
        \includegraphics[width=0.6\textwidth]{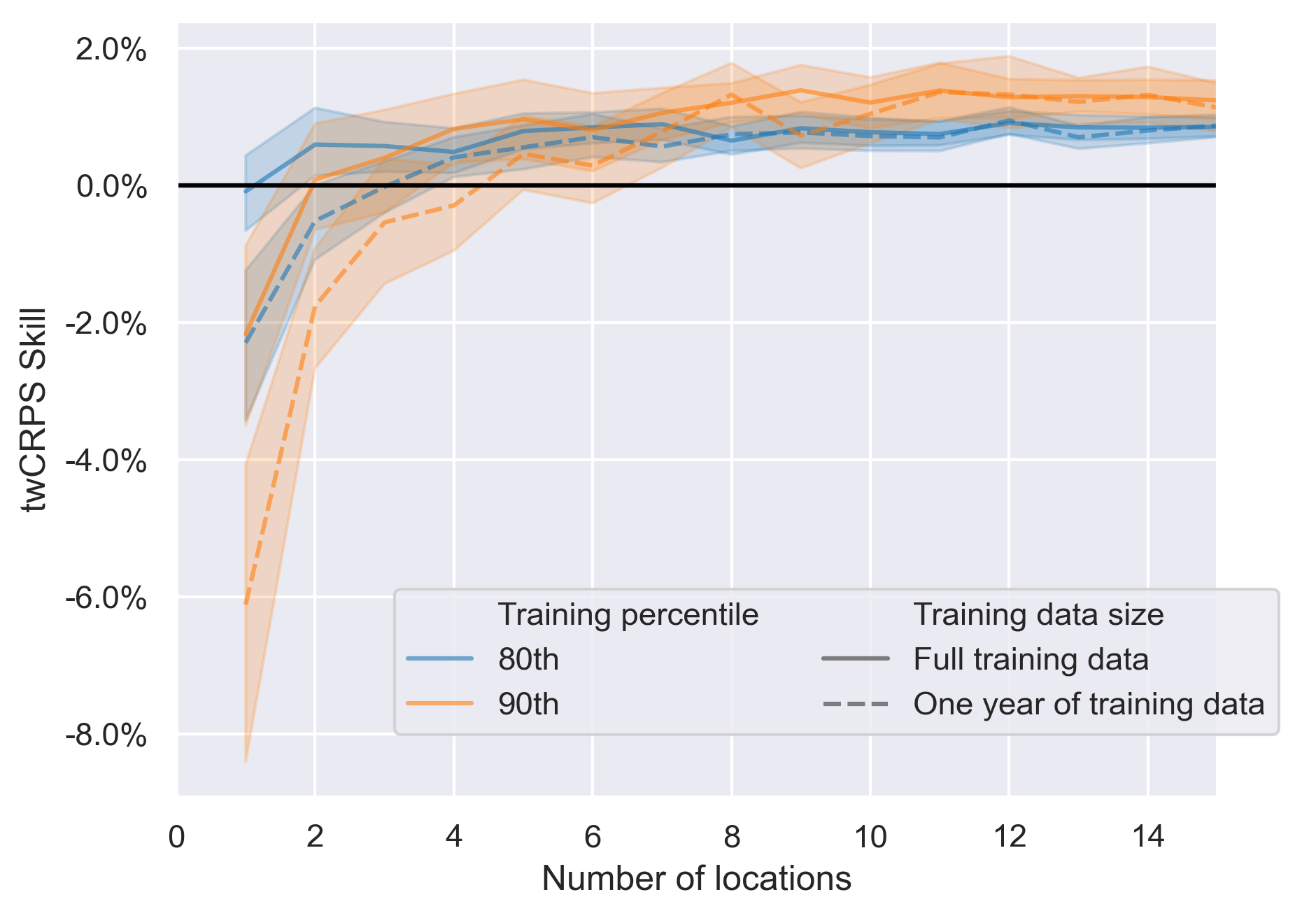}
        \caption{twCRPS skill score (over the CRPS-trained model) as a function of the number of locations included in the training set, for models trained on the full training dataset, as well as on only one year of training data. The bands represent 95\% confidence intervals obtained through bootstrapping, whilst the lines correspond to the mean score. Truncated normal predictive distribution and 80th as well as 90th percentile thresholds.}
        \label{fig:training_data_dep}
\end{figure}

\section{Forecast combinations}
\label{sec:combinations}
In section \ref{sec:first_results} we found a trade-off between twCRPS improvement and CRPS loss, when training using the twCRPS. A forecast distribution, it seems, can be very good for extremes but sub-optimal for values in the body of the climatological distribution. This can pose challenges for application. In the following we study two ways of combining the CRPS and twCRPS objectives to enable forecasters to balance the trade-off and adjust CRPS and twCRPS performance to their preferences. First we consider training EMOS models using weighted score combinations: 
\begin{equation}
\label{eq:weighted_training}
    \crps(F, y) + \gamma \cdot \twcrps_\tau(F,y).
\end{equation}
These combine training using a CRPS and twCRPS and correspond to a weighted CRPS with weighting function $w(x) = 1 + \gamma \cdot \mathbb{1}(x \geq \tau)$, originally proposed by \citet{vannitsem_2018_chapter_6} for the assessment of forecasts. Other, non-threshold-based, weighting functions might also be usable to balance body-tail forecast performance, however, indicator weighting has the advantage of closed-form solutions being available. In a second step we consider linear pooling between the CRPS-trained and twCRPS-trained model, to control twCRPS improvement and CRPS loss. We focus on the 90th percentile threshold for both truncated normal and truncated logistic distributions.

\subsection{Weighted combinations}

We consider training using weighted scores of the form given in equation \ref{eq:weighted_training}. In Figure \ref{fig:gamma} we show the CRPS skill over the CRPS-trained model (blue) and twCRPS skill over the twCRPS-trained model (orange) for different values of $\gamma$. For $\gamma = 0$ the model corresponds to the CRPS-trained model whilst for growing $\gamma$ the twCRPS has stronger influence. \\

Varying $\gamma$ seems to allow users to control the CRPS loss and twCRPS improvement, with as $\gamma$ grows the resulting model approaches the twCRPS improvement of the twCRPS-trained model and its CRPS loss. Interestingly however, even though the twCRPS for the model with $\gamma = 20$ is nearly at the level of the twCRPS-trained model (at least for the truncated normal distribution), the CRPS skill loss is still much better with only -1.58\% CRPS skill loss for the $\gamma = 20$ model against -3.64\% for the twCRPS-trained models in the truncated normal case and only -1.75\% against -2.85\% in the truncated logistic case (see section \ref{sec:first_results}). For cluster 4, as well as for larger values of $\gamma$ cluster 3 (and 2 in the truncated logistic cases) weighted training can even improve upon the ``pure" twCRPS-model. Weighted training essentially leads to a more robust model, avoiding some cases where CRPS performance falls off strongly through twCRPS training, whilst still capturing much of the twCRPS improvement. This can also be of importance when training with very high thresholds, for example non-location-specific thresholds. When there are little to no observations over the threshold, twCRPS-based training can lead to highly mis-specified EMOS models, whilst weighted training allows to strike a compromise between the twCRPS and CRPS objectives (not shown). \\

The parameter $\gamma$  presents a ``hyperparameter" that can be chosen by users to tailor a model's CRPS and twCRPS performance for specific applications. However, after training, changing the balance of twCRPS improvement and CRPS loss requires retraining of the post-processing model. Thus in the following we will consider direct combination of CRPS-trained and twCRPS-trained models.

\begin{figure}[h!]
        \centering
        \includegraphics[width=0.9\textwidth]{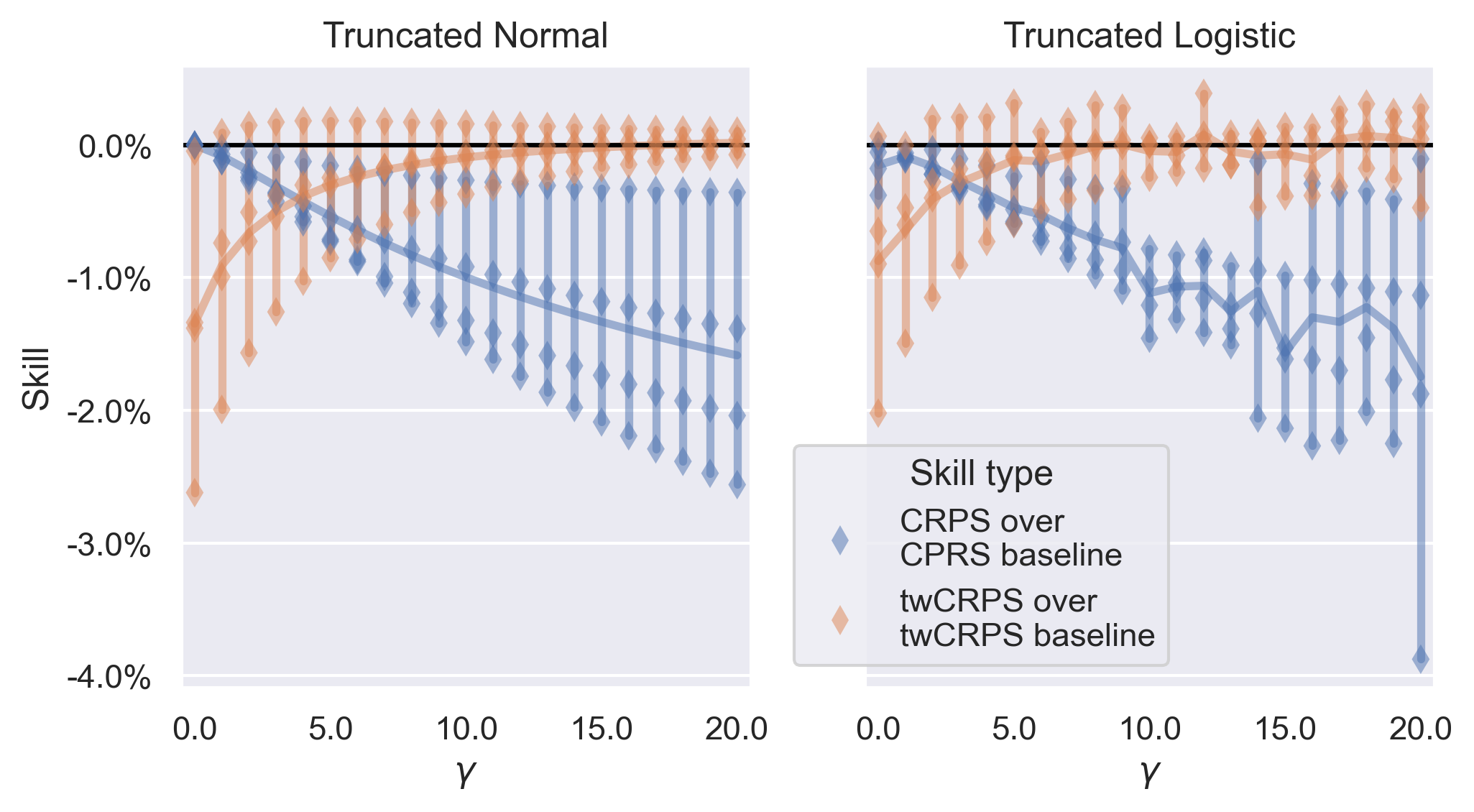}
        \caption{Weighted training: twCRPS skill score over the twCRPS-trained model (orange) and CRPS skill score over the CRPS-trained model (blue) as a function of $\gamma$ for both truncated normal (left) and truncated logistic (right) predictive distributions. Each diamond represents one cluster, whilst the lines correspond to the mean score.}
        \label{fig:gamma}
\end{figure}

\subsection{Linear pool}

The $\gamma$ parameter seems to ``interpolate" between the CRPS-trained and twCRPS-trained model. We now study combining these models directly using their predictive CDFs. More concretely we study the following convex combination of the predictive distribution functions:
\begin{align}
\label{eq:linear_pool}
    F_{\text{pool}} = \lambda F_\crps + (1-\lambda)F_\twcrps 
\end{align}
This is also known as linear pool \citep{gneiting_combining_2013, baran_combining_2018} and the canonical case $\lambda = 1/2$ corresponds to simple CDF averaging. The linear pool as distributional mixture provides a more flexible predictive distribution. It has some drawbacks, for example \citet{gneiting_combining_2013} show that it necessarily increases dispersion. More advanced versions such as the spread-adjusted linear pool have been proposed, however, we defer their study for future work.\\

We study the twCRPS skill over the twCRPS-trained model and the CRPS skill over the CRPS-trained model, whilst varying $\lambda$ in Figure \ref{fig:lambda}. The $\lambda$ parameter allows to strike a balance between the CRPS (being best for the CRPS-trained model and worst for the twCRPS-trained model) and the twCRPS (vice-versa). The curves are somewhat similar to the ones obtained by weighted training. As one can see the twCRPS improvements and CRPS losses have however very different ``steepness", meaning that for example a model with $\lambda = 0.6$ still has most of the twCRPS improvements, but controls the majority of the CRPS loss.\\

\begin{figure}[h!]
        \centering
        \includegraphics[width=0.9\textwidth]{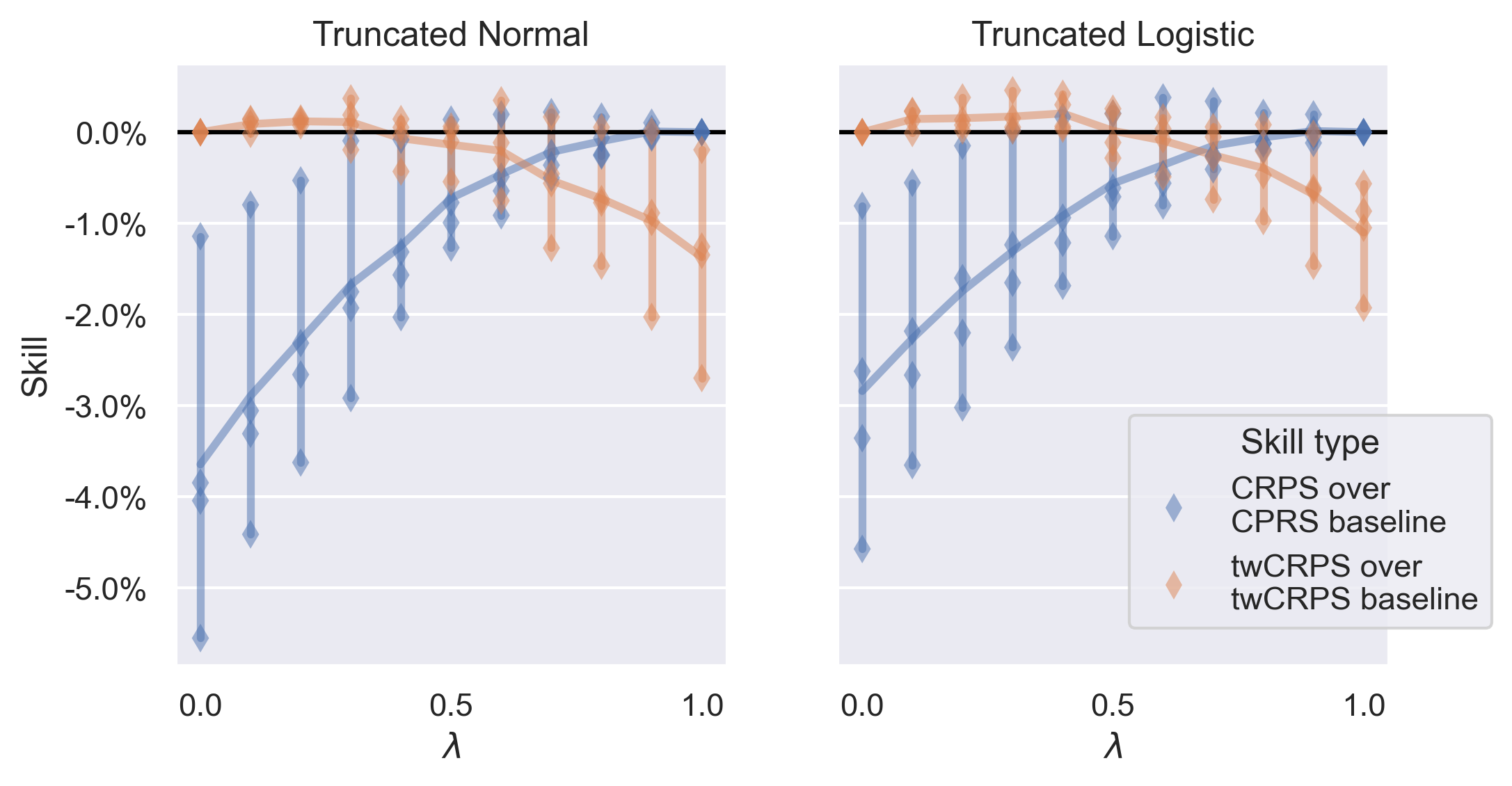}
        \caption{Linear pool: twCRPS skill score over the twCRPS-trained model (orange) and CRPS skill score over the CRPS-trained model (blue) as a function of $\lambda$ for both truncated normal (left) and truncated logistic (right) distributions. Each diamond represents one cluster, whilst the lines correspond to the mean score.}
        \label{fig:lambda}
\end{figure}

We investigate the example case $\lambda = 0.6$ for the truncated normal distribution. Figure \ref{fig:combination} shows the twCRPS and CRPS skill (relative to the CRPS-trained model) for both the twCRPS-trained and the linear pool model. As one can see the linear pool model keeps much of the twCRPS improvement, but controls the CRPS loss much stronger, with the spread across clusters being much smaller. Especially, for the CRPS performance, this will in part be due to the increased dispersion when linearly pooling forecasts. Pooling CRPS-trained and twCRPS-trained models seems to be able to somewhat alleviate the CRPS-twCRPS trade-off.  \\

\begin{figure}[h!]
        \centering
        \includegraphics[width=0.6\textwidth]{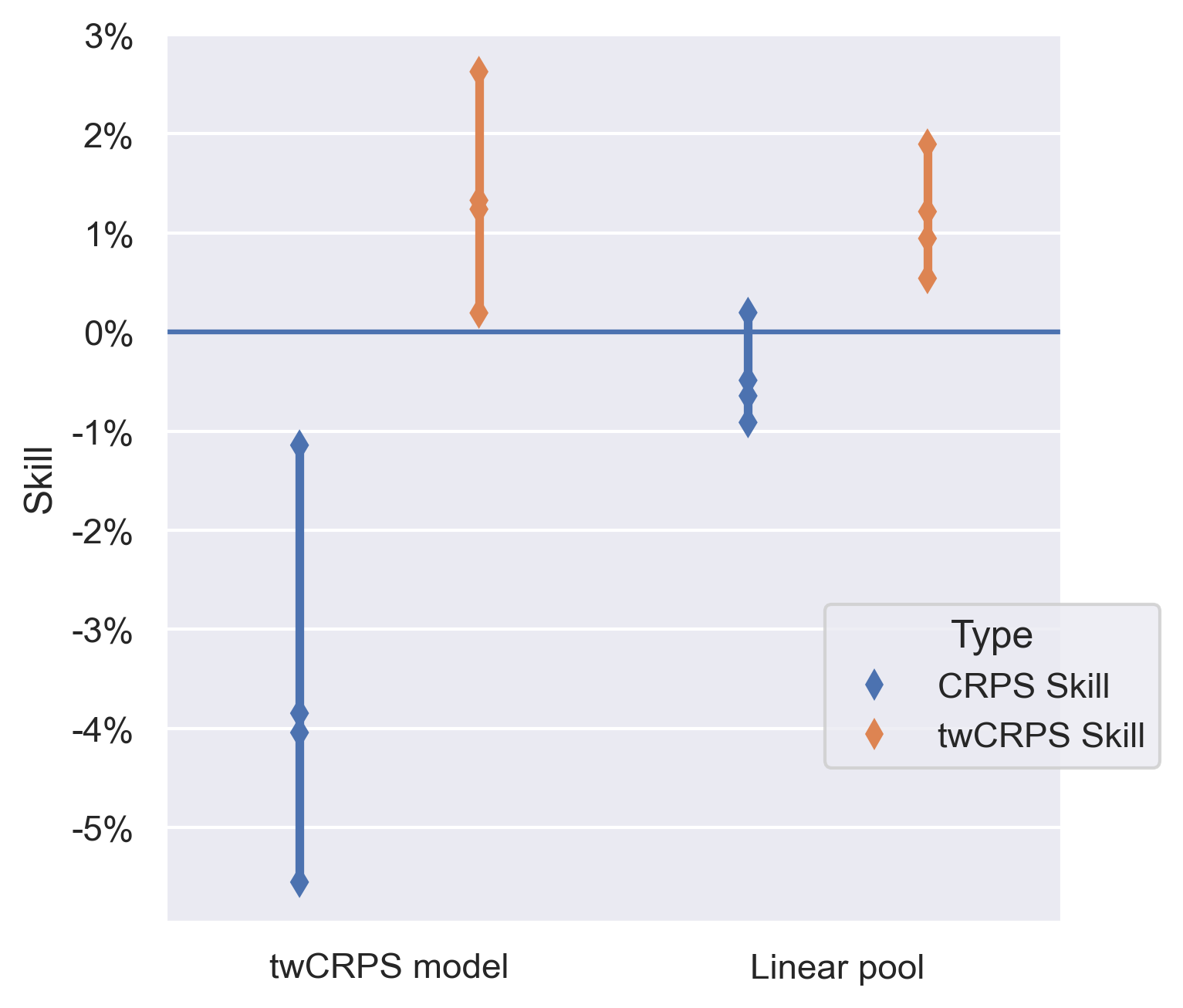}
        \caption{Linear pool: twCRPS skill score (orange) and CRPS skill score (blue), both relative to the CRPS-trained model, for the twCRPS and the linear pool model with $\lambda = 0.6$. Truncated normal predictive distribution, 90th percentile threshold.}
        \label{fig:combination}
\end{figure}

For different applications different choices of $\lambda$ will be relevant as users might be able to justify different CRPS losses. It might be possible to choose the $\lambda$ parameter in a data driven way, possibly based on the ensemble mean or external covariates such as weather patterns \citep{allen_accounting_2021, spuler_identifying_2024, spuler_learning_2025}. Simple experiments parametrizing $\lambda$ in terms of the distance of the ensemble mean to the threshold (a possible measure of the probability of exceedance), however did not show substantial skill over the ad-hoc choice of for example $\lambda = 0.6$.\\ 

Compared to the weighted training the linear pool directly interpolates between the two models, which means that for low values of $\lambda$ it is able to capture more of the twCRPS improvement, but also has much stronger CRPS loss. The latter is maximal with -1.58\% for $\gamma = 20$ for the weighted training, whilst reaching -3.64\% for $\lambda = 0$ in the linear pool, truncated normal case. By choosing adequate values of $\lambda$ and $\gamma$ however it is possible to construct models with very similar twCRPS and CRPS performance using both the linear pool and weighted training. Which one is preferable for applications might therefore depend on individual forecaster's considerations, most notably whether changing the twCRPS/CRPS balance might be required during application. This would need retraining in the case of weighted training, making the linear pool preferable. For choosing $\gamma$ or $ \lambda$ in applications there might be requirements on CRPS/twCRPS performance necessitating specific choices. Otherwise, we recommend an ad-hoc choice of a value that meets the objectives.

\section{Interpreting the training impact of the twCRPS}
\label{sec:synth_ex}
After experiments studying the training of post-processing models using the threshold-weighted CRPS we now consider a number of synthetic experiments to understand the training impact of the twCRPS better in an idealized system. We focus on inferring tail properties from simulated data and on extreme value distributions, which have been used for various post-processing applications \citep[see for example][]{baran_log-normal_2015, thorarinsdottir_assessing_2016}, to highlight how using a twCRPS can focus parameter inference on parameters controlling the distribution tail.\\

\subsection{Background: Tail structures and proper scoring rules}

\citet{taillardat_evaluating_2023} argue that the mean CRPS and mean weighted CRPS cannot distinguish between non-tail-equivalent forecasts. They show that for any random variable $Y$ and any $\nu > 0$ it is always possible to construct a non-tail-equivalent random variable $X$ (meaning a random variable with different tail index and thus differing tail behavior/return levels for extremes), such that:
\begin{align}
    \label{eq:wCRPS_nu}
    | \mathbb{E}_{y \sim Y}\left[ \wcrps(F_y, y) \right] - \mathbb{E}_{y \sim Y} \left[ \wcrps(F_x, y)\right]| \leq \nu
\end{align}
where $\wcrps$ is the weighted CRPS with any weighting function $w$.  \citet{brehmer_why_2019} generalise this and show that this is not only a property of the weighted CRPS, but rather all proper scoring rules fail to distinguish between tail equivalent forecasts in the sense of \citet{taillardat_evaluating_2023}. As \citet{brehmer_why_2019} conclude this casts doubt on the ability of proper scoring rules to distinguish between tail regimes. In practical applications, however, proper scores such as the log-score (negative log-likelihood) are often successfully used to fit distributions to extremes and to distinguish between tail structures. This is supported by the theory of M-estimation, of which scoring rule inference is a special case \citep{dawid_geometry_2007}, providing consistency guarantees on parameter estimates. Thus the authors of this work believe that the arguments presented in \citet{taillardat_evaluating_2023} and \citet{brehmer_why_2019} do not carry conclusive practical implications. This is especially as \citet{taillardat_evaluating_2023}  do not provide bounds on the constructed non-tail-equivalent random variables $X$ that show that these will influence statistical analysis in practice \footnote{More precisely \citet{taillardat_evaluating_2023} construct their non-tail-equivalent random variable as $X := Y \cdot \mathbb{1}(Y \leq u) + \left( Z + u \right) \cdot \mathbb{1}(Y > u)$, where $Z$ is a random variable with potentially very different tail behavior from $Y$. However, the authors give no ($\nu$-dependent) bounds on $u$. $u$ might therefore be very large and even unobservable in practice. Thus, although it is always possible to construct a non-tail-equivalent random variable it is not clear whether in practical applications $X$ will be markedly different from $Y$. Furthermore, the weighted CRPS in equation \ref{eq:wCRPS_nu} is assumed to be fixed. When using a twCRPS, however, the threshold can be adapted to the situation and can be set to $\tau > u$, which would effectively center inference onto $Z$ following Lemma \ref{lemma_sam_allen}.}. Nonetheless, we believe \citet{taillardat_evaluating_2023, brehmer_why_2019} highlight an interesting question about the ability of the weighted and especially threshold-weighted CRPS to distinguish between tail regimes. We believe that there is value in studying this question empirically, thus in the following we consider a number of synthetic examples estimating distributions with different tail indices from samples.


\subsection{Synthetic experiments}

We will consider the following mixture of generalized Pareto distribution ($GPD$) and normal distribution ($N$):
\begin{align}
\label{synth_model}
    0.4 \cdot N(\mu, 1) + 0.6 \cdot GPD(\xi),
\end{align}
where 
\begin{align*}
    F_{GPD}(x | \xi) = \begin{cases}
1 - \left(1+ \xi x\right)^{-1/\xi} & \text{for }\xi \neq 0, \\
1 - \exp \left(-x\right) & \text{for }\xi = 0.
\end{cases}
\end{align*}
In this mixture the $\mu$-parameter controls the location of the distribution body, whilst $\xi$ controls the tail structure. We fix the true $\mu = 2$ and $\xi = 0.5$. In the following, we will consider experiments estimating both parameters using CRPS and twCRPS from 1000 samples of the true model (``training samples"). We consider a grid of candidate values $\hat \mu, \hat \xi$ and for each of these we sample 250 values from the model in equation \ref{synth_model} and compute the sample average CRPS or twCRPS between those samples from the ``candidate model" and the 1000 samples from the true model (the ``training samples").\\

First we fix $\mu = 2$ and only estimate the $\xi$ parameter, by varying it between 0.25 and 0.75, shown in Figure \ref{fig:synth_ex1}. All three estimation methods, using CRPS, twCRPS with a threshold based on the 90th percentile and twCRPS with a threshold based on the 95th percentile of the training distribution, seem to be able to stably identify the true tail parameter $\xi = 0.5$. Interestingly however the twCRPS with 90th or 95th percentile thresholds, which following lemma \ref{lemma_sam_allen} effectively censor out 90\% / 95\% of the data seems to be just as stable for estimating $\xi$ as the CRPS. The loss curves do not seem substantially more ``noisy" and exhibit similar curvature across the domain of candidate $\xi$ values.\\

\begin{figure}[ht]
        \centering
        \includegraphics[width=\textwidth]{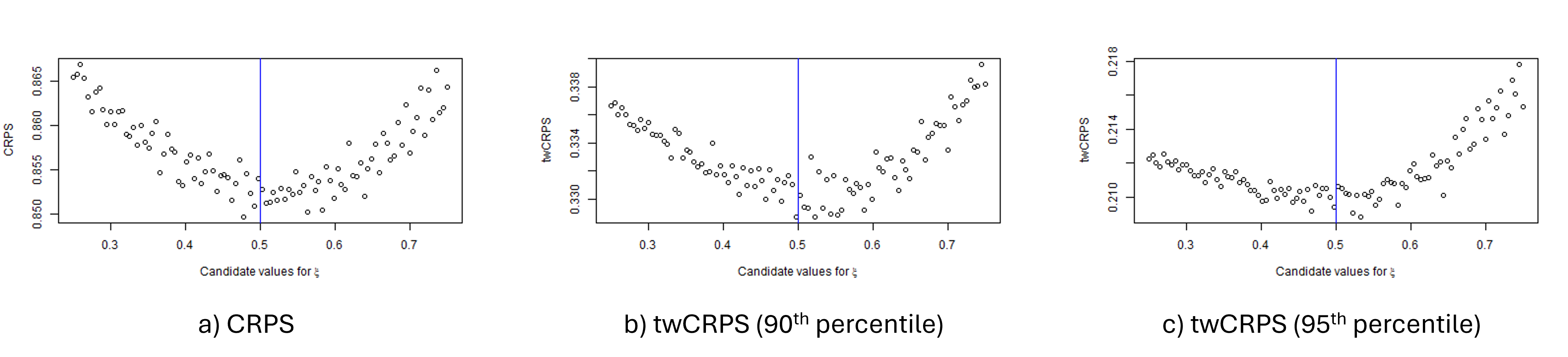}
        \caption{CRPS (a) and twCRPS for 90th (b) and 95th (c) percentile thresholds for different values of $\xi$, with the true $\xi = 0.5$ indicated by the blue vertical lines.}
    \label{fig:synth_ex1}
\end{figure}

In Figure \ref{fig:synth_ex2} we estimate both parameters $\mu$ and $\xi$ using the CRPS and twCRPS with a 90th percentile threshold. We plot the CRPS and twCRPS marginally, such that for each value of $\mu$ the CRPS/twCRPS for all candidate values of $\xi$ can be seen and vice-versa. We find that the CRPS is much sharper for the estimation of $\mu$ than it is for the estimation of $\xi$, where the true parameter can't be identified. This effect seems reversed for the twCRPS which is very inconclusive for $\mu$ estimation, but relatively sharp for the estimation of $\xi$. The twCRPS seems to focus the optimization on the parameter of interest, which controls the tail structure of the synthetic dataset, whilst the CRPS is concentrated on deviations in the body of the distribution. The type of proper score used for optimization can thus have an influence on the type of parameter focused during optimization. \\

\begin{figure}[ht]
        \centering
        \includegraphics[width=\textwidth]{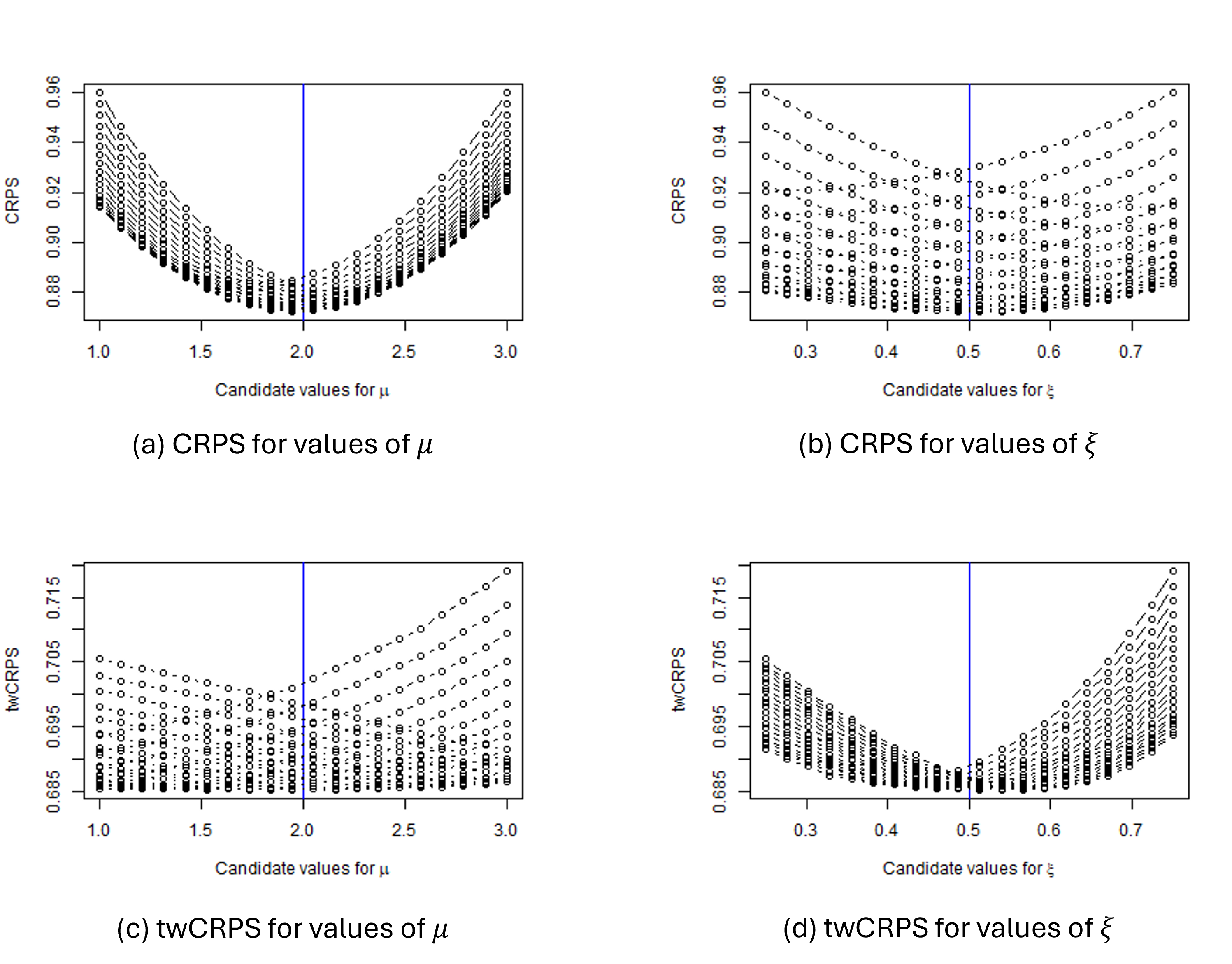}
        \caption{CRPS (top, a, b) and twCRPS with 90th percentile threshold (bottom, c, d) for different values of $\mu$ (left, a, c) and $\xi$ (right, b, d). The scatterpoints for each value of $\mu$ and $\xi$ correspond to all values of the respective other parameter, the lines connect the values of the other parameter being constant. The true $\mu$ and $\xi$ are indicated by the blue vertical lines. The same random seed is held constant for all samples for all candidate values of $\mu$ and $\xi$.}
    \label{fig:synth_ex2}
\end{figure}

The effect of the twCRPS of focusing inference on the tail structure of a distribution follows mathematically from lemma \ref{lemma_sam_allen}. By using the twCRPS with threshold $\tau$ one effectively evaluates the censored CDF using the CRPS, meaning that the body of the distribution only factors in via censoring (the probability mass at the threshold $\tau$ itself). This explains the sharp $\xi$ inference in Figure \ref{fig:synth_ex2}, whilst the $\mu$ inference is inconclusive. This process of using a CRPS for a censored distribution has some similarities with the usage of censored likelihood in extreme value theory \citep{huser_spacetime_2014} and it will be interesting to explore the usage of (weighted) proper scores for extreme value inference further in future work (see e.g. \citealt{yuen_crps_2014} for CRPS based inference in max stable processes and  \citealt{de_fondeville_high-dimensional_2018} for inference using the gradient scoring rule).\\

Finally, both the CRPS and twCRPS are proper scoring rules, which means that in expectation both should be minimal at the true model. In practical applications, however, the true model is rarely within the candidate model class, and when optimizing a score for model training one only computes the in-sample average score, not the expected one. Using a weighted scores thus allows to account for some model deficiencies, but also allows more efficient estimation by focusing parameter inference, even if the true model is within the model class, as Fig \ref{fig:synth_ex2} shows.

\section{Discussion and conclusions}
\label{sec:conclusion}
In this work we have introduced the usage of threshold-weighted scoring rules to train probabilistic post-processing models, tailored to extremes. This is an alternative to schemes based on weighting loss functions by outcomes which different researchers have been proposing in the literature, which however are problematic as they might encourage hedging \citep{lerch_forecasters_2017}. We have shown the positive effects of training with the twCRPS on forecasts' extreme event performance. We have identified a distribution body-tail trade-off which manifests as a trade-off between CRPS and twCRPS performance. We have shown that training using a twCRPS is possible for a variety of thresholds of interest and have shown that successful training depends on the training datasize used. We have developed two options for users to balance distribution body and extreme event performance of their forecasts: weighted sums of CRPS and twCRPS and the linear pool between the CRPS and twCRPS forecast. We developed a number of synthetic experiments to show the effect of the twCRPS and finally we have calculated closed-form twCRPS expressions for a number of distributions, which can be found in Appendix \ref{sec:closed_form} together with different characterizations of the twCRPS (Appendix \ref{sec:properties}) \\

We hope that the approaches developed in this article help researchers and operational users to improve forecasts of extreme events, events that are of particular societal relevance. The methodology allows users to choose specific thresholds for which they wish to improve forecasts, which can be relevant for many applications, for example, for airplanes or drones that are only able to operate within certain wind conditions. Given such a threshold, forecasters can use weighted scoring rules to improve distributional forecasts over the threshold. If they are not only interested in extremes they can use either linear pooling or weighted training to balance body and tail forecast performance. The  methodology presented has been demonstrated in an application for wind speed, however improving forecasts of extremes (values over a certain impact-relevant threshold) can also be relevant for many other meteorological variables such as temperature and related heat stress or precipitation and precipitation accumulation. The application of the method to these cases is left for future work, however to aid this research we provide closed form twCRPS expressions for many distributions in the Appendix.\\

There are a number of further research questions that follow from this work and should be explored. First it would be interesting to explore the usage of other weighted scores such as the proper quantile weighted score \citep{gneiting_comparing_2011} or the proper outcome weighted scores \citep{holzmann_focusing_2017}, censored likelihood scores \citep{de_punder_localizing_2023} and the usage of different, non-threshold-based, weighting functions when training models with weighted scores with emphasis on different parts of the distribution. \citet{allen_evaluating_2023, allen_weighted_2023} develop multivariate versions of the weighted scores, by showing that weighted scores fall into the class of kernel scores. In order to improve predictions of compound extremes -- extremes that can have the multiplicity of the impact of singular extremes \citep{zscheischler_typology_2020} -- investigating the usage of these scores to train models with potentially improved forecasting skill for compound extremes might be of interest. This is especially relevant as for example \citet{chen_novel_2022} show the positive effects of a kernelized loss function for improved deterministic predictions of wind speed extremes and similar approaches could be successful for probabilistic predictions. With more complex loss functions or scores, closed-form solutions of these might not be available, thus it might be of interest to investigate sample-based optimization of weighted scores. This might also be relevant when it comes to training or finetuning of generative deep learning models, which have shown great potential for post-processing applications \citep{chen_generative_2024}. As neural networks are able to adjust well to large numbers of covariates, this could be combined with research into what predictors for post-processing models are most relevant to improve extreme event forecasts. Research on this and the usage of (weighted) scoring rules for training generative neural networks is in progress by the authors. In addition it might be of interest to investigate forecast combination methods further, for combining distribution body and distribution tail forecasts. For example one could explore methods based on online-learning \citep{van_der_meer_crps-based_2024} or pointwise combinations \citep{berrisch_crps_2023} and see whether they can help to alleviate the trade-off between CRPS and twCRPS performance further. Alternatively, optimization methods such as multi-objective optimization could be explored to see whether they manage to balance the CRPS and weighted scoring rule objectives.\\

Secondly, much more sophisticated post-processing methods exist. Many of these are extensions of EMOS or still based on a closed-form distributional structure, for example so-called distributional regression networks (DRN, \citealt{rasp_neural_2018}) parametrize the parameters of a predictive distribution as neural networks. The authors are hopeful that the results in this article might also generalize to other EMOS-like model structures, including other predictive and possibly tail-adaptive distributions. However, we leave this for future research.  \\

Third and finally, it would be interesting to explore the usage of weighted scoring rules for the training of post-processing models in small training data situations, at longer lead times \citep{wessel_lead-time-continuous_2024}, seasonal to subseasonal time scales or even on climatic timescales for the bias adjustment of climate models \citep{maraun_bias_2016, spuler_ibicus_2024}, where uncertainties in the prediction of extreme events are even larger and thus the need for probabilistic predictions is even stronger. For these timescales it might also be of interest to explore the usage of extreme value theory informed post-processing methods \citep{friederichs_forecast_2012, lerch_comparison_2013, oesting_statistical_2017, velthoen_improving_2019, taillardat_forest-based_2019} and explore training of these using threshold-weighted scores, informed by the synthetic experiments in section \ref{sec:synth_ex}. This is especially relevant as parameter optimization methods, such as the ones explored here, are only one possibility to improve predictive performance of models, next to improvements in the model structure itself. It would be interesting to investigate this trade-off between model structure and estimation further and explore situations and models where threshold-weighted scores might be especially useful for estimation.\\

\textbf{Acknowledgements:} Jakob Wessel was supported during this work by the Engineering and Physical Sciences Research Council (EPSRC) under grant award number 2696930 and acknowledges support by The Alan Turing Institute’s Enrichment Scheme. The authors would like to thank Sam Allen for pointing out the connection between the threshold-weighted CRPS and the CRPS of a censored distribution, as well as input regarding closed-form expressions of the twCRPS. They are grateful to Jamie Kettleborough and Fiona Spuler for helpful discussions of the results and comments on the manuscript draft. Also, they would like to thank the three reviewers for comments and feedback that have substantially improved the manuscript. \\

\textbf{Code and Data availability:} The code for this study can be found at \url{https://github.com/jakobwes/Improving-probabilistic-forecasts-of-extreme-wind-speeds}. Unfortunately, the authors are unable to share the data, however this can be requested from the UK Met Office.

\pagebreak
\bibliographystyle{plainnat}
\bibliography{references} 

\clearpage
\appendix
\section{Characterizations and properties of the twCRPS}
\label{sec:properties}
In the following, to simplify the notation, we make frequent use of the chaining function $v_s$ in its canonical form defined as $v_s(x) := \max(s, x)$ for $x\in\mathbb{R}$ and $s\in\mathbb{R}$.\\

\textbf{Characterization 1}: equivalence of the twCRPS and the CRPS of the censored distribution (lemma \ref{lemma_sam_allen} in section \ref{sec:weighted_scoring_rules}).

\begingroup
\def\thetheorem{\ref{lemma_sam_allen}}
\begin{lemma}
    Let $F$ be a CDF and $y$ an observation. Then
    \begin{align}
        \twcrps_\tau (F, y) = \crps(\tilde F_\tau, v_\tau(y))
    \end{align}
    where $\tilde F_\tau$ is the CDF censored at $\tau$:
    \begin{align}
    \tilde F_\tau(x) = \begin{cases}
        0, & x < \tau, \\
        F(x), & x \geq \tau.
    \end{cases}
\end{align}
\end{lemma}
\addtocounter{theorem}{-1}
\endgroup
\begin{proof}[Proof]
The proof follows from the definitions of ${\tilde{F}}_\tau$ and $v_\tau$:
\begin{align*}
    \crps(\tilde F_\tau, v_\tau (y)) &= \int_{-\infty}^\infty \left[\tilde{F}_\tau(x) - \mathbb{1}\left\{x \geq v_\tau(y)\right\}\right]^2 dx \\
    &= \int_\tau^\infty \left[F(x) - \mathbb{1}\left\{x \geq y\right\}\right]^2 dx \\
    &= \twcrps_\tau (F, y)
\end{align*}
\end{proof}

\textbf{Characterization 2}: expression of the twCRPS through the CRPS and a correction term. 
\begin{lemma}
\label{lemma_computation}
Let $F$ be a CDF and $y$ an observation. Then
\begin{align}
    \twcrps_\tau(F, y) = \crps\left(F,v_\tau(y)\right) - \Delta_F 
\end{align}
where $\Delta_F$ is a correction term given as
\begin{align}
    \Delta_F := \int_{-\infty}^\tau \left[F(x)\right]^2 dx.
\end{align}
\end{lemma}
\begin{proof}[Proof (Lemma \ref{lemma_computation})]
The proof follows by considering integration limits in the CRPS and twCRPS definitions and the definition of $v_\tau$:
    \begin{align*}
        \crps(F, v_\tau (y)) - \Delta_F &= \int_{-\infty}^\tau \left[F(x)\right]^2 dx + \int_\tau^{v_\tau(y)} \left[F(x)\right]^2 dx \\ \nonumber &+ \int_{v_\tau(y)}^\infty \left[F(x) - 1\right]^2 dx - \int_{-\infty}^\tau \left[F(x)\right]^2 dx \\
        &= \int_\tau^{v_\tau(y)} \left[F(x)\right]^2 dx + \int_{v_\tau(y)}^\infty \left[F(x) - 1\right]^2 dx \\
        &= \twcrps_\tau(F, y).
    \end{align*}        
\end{proof}

\textbf{Property 1}: The twCRPS is stable under censoring at a point below the threshold $\tau$.

\begin{lemma}
    \label{lemma_censoring}
    Let $F$ be a CDF and ${\tilde{F}}_t$ be the CDF censored at $t$:
    \begin{align}
        {\tilde{F}}_t(x) = \begin{cases}
        0, & x < t, \\
        F(x), & x \geq t.
    \end{cases}
    \end{align}

    Then for $t \leq \tau$:
    
    \begin{align}
        \twcrps_\tau ({\tilde{F}}_t, y) = \twcrps_\tau (F, y)
    \end{align}
\end{lemma}
\begin{proof}
    This follows from Lemma \ref{lemma_sam_allen}.
\end{proof}

\textbf{Property 2}: The lower-tail twCRPS can be expressed in terms of the CRPS and the upper-tail twCRPS.

\begin{lemma}
    The lower-tail twCRPS (for values below instead of above a threshold),
    \begin{align}
        \twcrps_\tau^{\text{l}}(F,y) :=  \int_{-\infty}^{\tau}  \left[F(x) - \mathbb{1}\{x \geq y\}\right]^2 dx 
    \end{align}
   can be expressed in terms of the upper-tail twCRPS as
    \begin{align}
        \twcrps_\tau^{\text{l}}(F,y) + \twcrps_\tau^{\text{u}}(F,y) = \crps(F,y) 
    \end{align}
	where $\twcrps_\tau^{\text{u}}$ is equal to $\twcrps_\tau$ as used throughout the paper.
\end{lemma}
\begin{proof}
    This follows when considering the integration limits.
\end{proof}

\section{Closed-form expressions for the twCRPS}
\label{sec:closed_form}

We here provide a list of closed-form expressions for the twCRPS for a couple of standard distributions often used in the context of statistical post-processing. Rather than aiming at an exhaustive list, we restrict ourselves to generic cases of the most important distributions. All of the formulae have been verified against numerical integration. In the limit of the threshold $\tau$ moving towards the lower end of the distribution, known expressions for the CRPS are recovered. \\

Note that the twCRPS is stable under censoring (see Lemma \ref{lemma_censoring}), meaning that all of the closed-form expressions below also apply to censored versions of the distributions if censoring at a point below the threshold $\tau$.\\

We again make use of the chaining function $v_s(x)=\max(s,x)$ for $x\in\mathbb{R}$ and $s\in\mathbb{R}$.\\

\subsection{Distributions supported on the real line}

\subsubsection{Laplace distribution}
twCRPS expression:
\begin{gather}
    \twcrps_\tau \left(F, y\right) =  \begin{cases}
        |v_\tau(y)| -\frac{3}{4} + \exp(-|v_\tau(y)|) - \frac{1}{8} \exp(2\tau), & \tau < 0,\\
        v_\tau(y) - \tau + \exp(-v_\tau(y)) + \frac{1}{8} \exp(-2\tau) - \exp(-\tau), & \tau \geq 0,\\
    \end{cases} \\
    \twcrps_\tau\left(F_{\mu, \sigma}, y\right) = \sigma \, \twcrps_{(\tau-\mu)/\sigma}\left(F,\frac{y-\mu}{\sigma}\right),
\end{gather}
where $F(x)=\frac{1}{2}+\frac{1}{2}{\rm sgn}(x)[1-\exp(-|x|)]$ is the standard Laplace CDF, and $F_{\mu,\sigma}(x)=F\left(\frac{x-\mu}{\sigma}\right)$ with location 
parameter $\mu\in\mathbb{R}$ and scale parameter $\sigma >0$.\\

Source: derived by the authors.

\subsubsection{Logistic distribution}
twCRPS expression:
\begin{gather}
    \twcrps_\tau\left(F, y\right) = F(\tau) -1 +\log\frac{1-F(\tau)}{1-F(v_\tau(y))} -\log F(v_\tau(y)), \\
    \twcrps_\tau\left(F_{\mu, \sigma}, y\right) = \sigma \, \twcrps_{(\tau-\mu)/\sigma}\left(F,\frac{y-\mu}{\sigma}\right),
\end{gather}
where $F(x)=1/[1+\exp(-x)]$ is the standard logistic CDF, and $F_{\mu,\sigma}(x)=F\left(\frac{x-\mu}{\sigma}\right)$ with location parameter $\mu\in\mathbb{R}$ and scale
parameter $\sigma >0$.\\

Source: This formula is given by \citet{allen_incorporating_2021}.

\subsubsection{Normal distribution}
twCRPS expression:
\begin{gather}
\begin{align}
    \twcrps_\tau\left(\Phi, y\right) = &-\tau \Phi^2(\tau) + v_\tau(y) \left[2\Phi(v_\tau(y))-1\right] \\ \nonumber
    &+ 2 \left[\varphi(v_\tau(y))-\varphi(\tau)\Phi(\tau)\right]
	- \frac{1}{\sqrt{\pi}} \left[1-\Phi(\sqrt{2}\tau)\right],
\end{align}\\
    \twcrps_\tau\left(F_{\mu, \sigma}, y\right) = \sigma\,\twcrps_{(\tau-\mu)/\sigma}\left(\Phi,\frac{y-\mu}{\sigma}\right),
\end{gather}
where $\varphi$ and $\Phi$ are the standard normal PDF and CDF, respectively, and $F_{\mu,\sigma}(x)=\Phi\left(\frac{x-\mu}{\sigma}\right)$ with location parameter $\mu\in\mathbb{R}$
and scale parameter $\sigma > 0$.\\

Source: This formula is given by \citet{wessel_lead-time-continuous_2024}.

\subsubsection{Student's $t$ distribution}
twCRPS expression:
\begin{gather}
    \begin{align}
    \twcrps_{\tau}\left(F_{\nu}, y\right) = &  - \tau F_\nu^2(\tau)+ v_\tau(y) \left[2F_\nu(v_\tau(y)) - 1\right] \nonumber \\
                  &	+ 2\left[F_\nu(\tau)G_\nu(\tau) - G_\nu(v_\tau(y))\right]  - \bar{B}_\nu \left[1 - H_\nu(\tau)\right],
\end{align}\\
    \twcrps_\tau\left(F_{\nu,\mu,\sigma}, y\right) = \sigma \, \twcrps_{(\tau-\mu)/\sigma}\left(F_\nu,\frac{y-\mu}{\sigma}\right),
\end{gather}
where $F_\nu$ is the CDF of the Student's $t$ distribution with $\nu >0$ degrees of freedom, and $F_{\nu,\mu,\sigma}(x)=F_\nu\left(\frac{x-\mu}{\sigma}\right)$ with
location parameter $\mu\in\mathbb{R}$ and scale parameter $\sigma >0$. Here, $\bar{B}_\nu$ and the functions $G_\nu$ and $H_\nu$ are defined as in \citet{jordan_evaluating_2019}.\\

Source: derived by the authors. This follows from the CRPS expression for the censored Student's $t$ distribution in \citet{jordan_evaluating_2019}.

\subsection{Distributions with non-negative support}

\subsubsection{Exponential distribution}
twCRPS expression for $\tau\geq 0$:
\begin{gather}
    \twcrps_\tau\left(F_\lambda, y\right) =  v_\tau(y) - \tau +\frac{2}{\lambda}\left[\exp(-\lambda v_\tau(y))-\exp(-\lambda\tau)\right] + \frac{1}{2\lambda}\exp(-2\lambda\tau),
\end{gather}
where $F_\lambda$ is the CDF of the exponential distribution with rate parameter $\lambda >0$.\\

Source: derived by the authors.

\subsubsection{Gamma distribution}
twCRPS expression for $\tau\geq 0$:
\begin{align}
    \twcrps_\tau \left(F_{\alpha,\beta}, y\right) = & -\tau F_{\alpha,\beta}^2(\tau) +v_\tau(y)[2F_{\alpha,\beta}(v_\tau(y))-1] \nonumber \\
& +\frac{\alpha}{\beta}[1-F_{\alpha,\beta}^2(\tau)+2F_{\alpha,\beta}(\tau)F_{\alpha+1,\beta}(\tau)-2F_{\alpha+1,\beta}(v_\tau(y))] -\frac{1-F_{2\alpha,\beta}(2\tau)}{\beta B\left(\frac{1}{2},\alpha\right)},
\end{align}
where $F_{\alpha,\beta}$ is the CDF of the gamma distribution with shape parameter $\alpha >0$ and rate parameter $\beta >0$, and $B$ denotes the beta function.\\

Source: derived by the authors.

\subsubsection{Log-logistic distribution}
twCRPS expression for $0<\sigma<1$ and $\tau\geq 0$:
\begin{equation}
\begin{split}
\twcrps_\tau\left(F_{\mu, \sigma}, y\right) = & -\tau F_{\mu,\sigma}^2(\tau) + v_\tau(y)\left[2F_{\mu, \sigma}(v_\tau(y)) - 1\right] \\
   & + 2 \exp(\mu) \left[B_{\rm u}\left(F_{\mu,\sigma}(v_\tau(y)); 1 + \sigma, 1 - \sigma\right) -B_{\rm u}\left(F_{\mu,\sigma}(\tau);2+\sigma,1-\sigma\right)\right],
\end{split}
\end{equation}
where $F_{\mu, \sigma}$ is the log-logistic CDF with log-location parameter $\mu\in\mathbb{R}$ and log-scale parameter $\sigma >0$, and $B_{\rm u}$ denotes the upper incomplete beta function. The alternative parameterization of the log-logistic distribution in terms of scale parameter $\alpha >0$ and shape parameter $\beta >0$ follows from $\mu = \log\alpha$ and $\sigma = 1/\beta$.\\

Source: derived by the authors.

\subsubsection{Log-normal distribution}
twCRPS expression for $\tau > 0$:
\begin{equation}
    \begin{split}
        \twcrps_\tau\left(F_{\mu, \sigma}, y\right) = & -\tau F^2_{\mu, \sigma}(\tau) + v_\tau(y) [2 F_{\mu, \sigma}(v_\tau(y)) -1] \\ 
        & + 2 \exp\left(\mu + \frac{1}{2} \sigma^2 \right) \left[1-F_{\mu + \sigma^2, \sigma} (v_\tau(y)) - I\left(\frac{\log\tau -\mu}{\sigma} -\sigma, \sigma\right)\right],
    \end{split}
\end{equation}
where $F_{\mu,\sigma}$ is the CDF of the log-normal distribution with log-location parameter $\mu\in\mathbb{R}$ and log-scale parameter $\sigma >0$.

Here, $I$ is the integral function
\begin{equation}
\label{def_I}
    I(\alpha, \beta) = \int_\alpha^\infty \varphi(u)\, \Phi(u + \beta) \, du,\hspace{1.5cm}\alpha\in\mathbb{R},\hspace{0.5cm}\beta\in\mathbb{R},
\end{equation}
with $\varphi$ and $\Phi$ denoting the standard normal PDF and CDF, respectively. We discuss the evaluation of this integral function in section \ref{sec:I}. \\

Source: derived by the authors.

\subsection{Distributions with flexible support}

\subsubsection{Continuous uniform distribution on $[a, b]$}
twCRPS expression:
\begin{gather}
    \twcrps_\tau\left(F, y\right) = \frac{1}{3}\left(1-\tau^3\right) +v_\tau^2(y) -v_\tau(y), \qquad 0\leq\tau\leq 1, \qquad 0\leq y\leq 1,\\
	\twcrps_\tau\left(F_a^b, y\right) = (b-a) \, \twcrps_{(\tau-a)/(b-a)}\left(F,\frac{y-a}{b-a}\right), \qquad a\leq\tau\leq b, \qquad a\leq y\leq b,
\end{gather}
where $F$ is the CDF of the standard uniform distribution, and $F_a^b(x)=F\left(\frac{x-a}{b-a}\right)$ with $-\infty < a < b<\infty$.\\

Source: derived by the authors.

\subsubsection{Generalized Pareto distribution}
twCRPS expression for $\xi < 1$: 
\begin{gather}
\begin{align}
\twcrps_\tau \left(F_\xi,y\right) = & v_\tau(y) -\tau +\frac{2}{1-\xi}\left\{{\left[1-F_\xi(v_\tau(y))\right]}^{1-\xi} -{\left[1-F_\xi(\tau)\right]}^{1-\xi}\right\} \nonumber \\
& +\frac{1}{2-\xi}{\left[1-F_\xi(\tau)\right]}^{2-\xi},\qquad\qquad \tau\geq 0,
\end{align}\\
\twcrps_\tau\left(F_{\xi,\mu,\sigma},y\right) = \sigma \, \twcrps_{(\tau-\mu)/\sigma}\left(F_\xi,\frac{y-\mu}{\sigma}\right), \qquad\qquad \tau\geq\mu,
\end{gather} 
where $F_\xi$ is the CDF of the generalized Pareto distribution with shape parameter $\xi\in\mathbb{R}$, zero location and unit scale, and $F_{\xi,\mu,\sigma}(x)=F_\xi\left(\frac{x-\mu}{\sigma}\right)$ with location parameter $\mu\in\mathbb{R}$ and scale parameter $\sigma >0$.\\

Source: derived by the authors.

\subsubsection{Truncated logistic distribution on $[a, b]$}
twCRPS expression for $a\leq\tau\leq b$ and $a\leq y\leq b$:
\begin{gather}
\begin{align}
    \twcrps_\tau \left(F_a^b, y\right) = & \frac{1}{[F(b) - F(a)]^2} \left\{ F(\tau) - F(b) + F^2(a) \log \frac{F(v_\tau(y))}{F(\tau)} \right. \nonumber \\
    & \left. + [1-F(a)]^2 \log \frac{1-F(\tau)}{1-F(v_\tau(y))} + F^2(b) \log \frac{F(b)}{F(v_\tau(y))} \right. \nonumber \\
    & \left. + [1-F(b)]^2 \log \frac{1-F(v_\tau(y))}{1-F(b)} \right\}, 
\end{align}\\
    \twcrps_\tau\left(F_{a,\mu,\sigma}^b, y\right) = \sigma \, \twcrps_{(\tau-\mu)/\sigma}\left(F_{(a-\mu)/\sigma}^{(b-\mu)/\sigma},\frac{y-\mu}{\sigma}\right),
\end{gather}
where $F(x)=1/[1+\exp(-x)]$ is the standard logistic CDF, $F_a^b$ is the standard logistic CDF truncated to $[a,b]$ with $-\infty < a < b <\infty$, and $F_{a,\mu,\sigma}^b(x)=F_{(a-\mu)/\sigma}^{(b-\mu)/\sigma}\left(\frac{x-\mu}{\sigma}\right)$ with location parameter
$\mu\in\mathbb{R}$ and scale parameter $\sigma >0$.\\

For only left truncation ($b\to\infty$), we have:
\begin{gather}
\begin{align}
    \twcrps_\tau \left(F_a^\infty, y\right) = & - \frac{1-F(\tau)}{[1-F(a)]^2} +\log\frac{1-F(\tau)}{1-F(v_\tau(y))} - \frac{F^2(a)}{[1-F(a)]^2}\log F(\tau) \nonumber \\
	& - \frac{1+F(a)}{1-F(a)} \log F(v_\tau(y)).
	\end{align}
\end{gather}

Source: For only left truncation, this formula is given by \citet{allen_incorporating_2021}; the general case of truncation to $[a,b]$ is derived by the authors.

\subsubsection{Truncated normal distribution on $[a, b]$}
twCRPS expression for $a\leq\tau\leq b$ and $a\leq y\leq b$:
\begin{equation}
\begin{split}
\twcrps_\tau\left(F_a^b, y\right) = & -\tau \left[\frac{\Phi(\tau) - \Phi(a)}{\Phi(b) - \Phi(a)}\right]^2 + v_\tau(y) \left[ 2\frac{\Phi(v_\tau(y)) - \Phi(a)}{\Phi(b) - \Phi(a)} -1 \right] \\ 
& + \frac{2}{\Phi(b) - \Phi(a)} \left[ \varphi(v_\tau(y)) - \varphi(\tau) \frac{\Phi(\tau) - \Phi(a)}{\Phi(b) - \Phi(a)}\right] - \frac{1}{\sqrt{\pi}} \frac{\Phi(\sqrt{2}b) - \Phi(\sqrt{2}\tau)}{\left[\Phi(b) - \Phi(a)\right]^2},
\end{split}
\end{equation}
\begin{equation}
\twcrps_\tau\left(F_{a,\mu, \sigma}^b, y\right) = \sigma \, \twcrps_{(\tau-\mu)/\sigma}\left(F_{(a-\mu)/\sigma}^{(b-\mu)/\sigma},\frac{y-\mu}{\sigma}\right),
\end{equation}
where $F_a^b$ is the standard normal CDF truncated to $[a,b]$ with $-\infty < a< b<\infty$, and $F_{a,\mu,\sigma}^b(x)=F_{(a-\mu)/\sigma}^{(b-\mu)/\sigma}\left(\frac{x-\mu}{\sigma}\right)$ with location
parameter $\mu\in\mathbb{R}$ and scale parameter $\sigma >0$. Here, $\varphi$ and $\Phi$ denote the standard normal PDF and CDF, respectively.\\

The case of only left truncation ($b\to\infty$) follows from $\Phi\left(\frac{b-\mu}{\sigma}\right)=\Phi\left(\frac{\sqrt{2}(b-\mu)}{\sigma}\right)= 1$ for any $\mu\in\mathbb{R}$ and $\sigma >0$.\\

Source: This formula is given by \citet{wessel_lead-time-continuous_2024}.

\subsection{Some notes on the integral function $I(\alpha, \beta)$}
\label{sec:I}

In the derivation of the twCRPS for the log-normal distribution the integral function $I(\alpha, \beta)$ appears (equation(\ref{def_I})). 
It has the statistical interpretation as the probability $P(X > \alpha, Y < X+\beta)$ where $X$ and $Y$ are two independent standard normal random variables.
The integral function $I$ can be expressed in terms of Owen's $T$ function \citep{owen_tables_1956, owen_table_1980},
\begin{equation}
    T(h, a) = \frac{1}{2\pi} \int_0^a \frac{\exp\left[-\frac{1}{2} h^2 (1+x^2)\right]}{1+x^2},\qquad h\in\mathbb{R},\qquad a\in\mathbb{R},
\end{equation}
as
\begin{equation}
    I(\alpha, \beta) = \frac{1}{2}\left[\Phi\left(\frac{\beta}{\sqrt{2}}\right) - \Phi(\alpha)\right] + T\left(\alpha, 1+\frac{\beta}{\alpha}\right) 
	+ T\left(\frac{\beta}{\sqrt{2}}, 1+\frac{2\alpha}{\beta}\right) + \frac{1}{2} \mathbb{1}\{\alpha \beta < 0\},
\end{equation}
for $\alpha \neq 0$ and $\beta \neq 0$, and otherwise
\begin{align}
    I(\alpha, 0) &= \frac{1}{2}[1-\Phi^2(\alpha)],\\ 
	I(0, \beta) &= \Phi\left(\frac{\beta}{\sqrt{2}}\right)  - \frac{1}{2} \Phi^2\left(\frac{\beta}{\sqrt{2}}\right).
\end{align}
There are routines available in \texttt{R} and other languages for evaluating Owen's $T$ function based on various series expansions or numerical integration. 

Alternatively, we propose a fast and stable method for evaluating $I(\alpha,\beta)$ based on a recursive power series expansion, making use of (probabilist's) Hermite 
polynomials ${\{H_n\}}_{n=0}^\infty$. The integral function $I$ has the Taylor series with respect to the second argument
\begin{align}
    I(\alpha, \beta) = \int_a^\infty \varphi(u)\, \Phi(u+\beta)\, du = \sum_{n = 0}^\infty I_n
\end{align}
where
\begin{equation}
    I_n = \frac{J_n}{n!} \beta^n
\end{equation}
with
\begin{equation}
    J_0 = I(\alpha,0) = \frac{1}{2} [1-\Phi^2(\alpha)]
\end{equation}
and
\begin{equation}
    J_n = {\left.\frac{\partial^n I}{\partial\beta^n}\right|}_{(\alpha,0)} = (-1)^{n-1} \int_\alpha^\infty H_{n-1}(u)\, \varphi^2(u)\, du, \qquad\qquad n \geq 1 .
\end{equation}
Using integration by parts and properties of the Hermite polynomials yields the following recursive scheme:

Initialize
\begin{equation}
    I_0 =  \frac{1}{2} [1-\Phi^2(\alpha)], \quad I_1 = \frac{\beta}{2\sqrt{\pi}} [1-\Phi(\sqrt{2}\alpha)], \quad I_2 = -\frac{\beta^2}{4} \varphi^2 (\alpha),
\end{equation}
and
\begin{equation}
    C_1 = 0, \quad C_2 = - \frac{\beta^2}{4} \varphi^2(\alpha),
\end{equation}
then iterate:
\begin{align}
    C_n = & -\frac{\alpha \beta}{n} C_{n-1} - \frac{n-3}{(n-1)n}\beta^2 C_{n-2}, \qquad\qquad n \geq 3, \\
    I_n = & C_n - \frac{n-2}{2(n-1)n} \beta^2 I_{n-2}, \qquad\qquad n \geq 3.
\end{align}
In practice, the power series is truncated after a finite number of terms as soon as $|I_n|$ has become small enough. 
As the recursion may run in two branches of different magnitude involving even and odd indices $n$, respectively, we truncate the series after $N_*$ terms where $N_*$ is the 
smallest natural number such that $|I_{N_* -1}| < \eta$ and $|I_{N_*}| < \eta$ with, say, $\eta={10}^{-15}$.
An estimate of the absolute error in $I(\alpha,\beta)$ which turns out to be very close to the actual absolute error can be obtained as
\begin{align}
    \text{err} = \max \left(\eta, \varepsilon \max_{N = 0, \dots, N_*} \left| \sum_{n=0}^N I_n\right|\right),
\end{align}
where $\varepsilon$ is the smallest number on the machine, typically $\varepsilon =2^{-52}\approx 2.2\times {10}^{-16}$. 
In numerical experiments we find that $I(\alpha, \beta)$ is evaluated rather fast and accurately using the above 
recursion formulas for any $\alpha$ and about $|\beta| < 7$. Computation can be further accelerated and the range of tractable values of $\beta$ extended by invoking the identity
\begin{align}
    I(\alpha, \beta) = \Phi \left(\frac{\beta}{\sqrt{2}}\right) - \Phi(\alpha) + I\left(\frac{\beta}{\sqrt{2}}, \sqrt{2}\alpha\right)
\end{align}
which is derived using results from \citet{owen_table_1980}. In summary, our scheme allows to evaluate the twCRPS for the log-normal distribution fast ($N_\ast <50$) and 
virtually to machine precision for any practically relevant combination of parameters $\mu$, $\sigma$ and $\tau$.

\section{Additional figures}

\begin{table}[h]
\caption{80th and 90th percentile thresholds for each of the four clusters. The colors refer to Figure \ref{fig:locations}.}
\vspace{5mm}
\centering
\begin{tabular}{|l|l|c|l|l|}
\hline
Cluster    & Cluster interpretation         & Number of locations & 80th percentile & 90th percentile \\ \hline
1 (purple) & Inland locations               & 30                  & 4.1m/s          & 5.7m/s          \\
2 (green)  & More exposed inland locations  & 44                  & 5.7m/s          & 7.2m/s          \\
3 (red)    & Coastal locations              & 35                  & 7.7m/s          & 9.8m/s          \\
4 (blue)   & More exposed coastal locations & 15                  & 10.3m/s         & 12.3m/s         \\ \hline
\end{tabular}
\label{tab:percentiles}
\end{table}

\begin{figure}[h]
        \centering
        \includegraphics[width=0.6\textwidth]{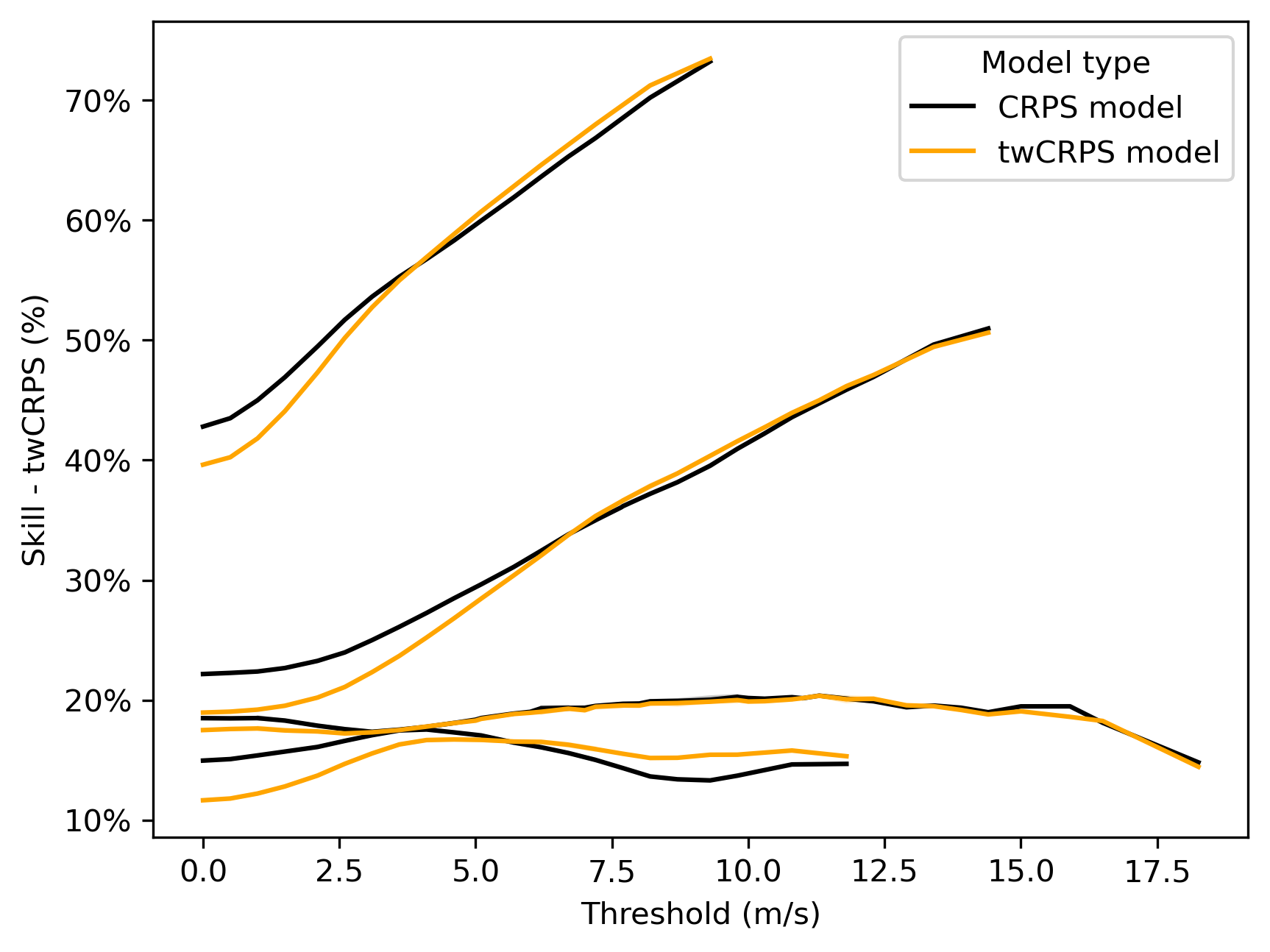}
        \caption{twCRPS skill for the CRPS-trained and twCRPS-trained EMOS model (truncated normal predictive distribution, 90th percentile threshold), compared to the raw ensemble for a range of evaluation thresholds between 0th and 99th percentile of the training distribution. Each line corresponds to one cluster. Note that due to the large percentage differences the relative improvement of the twCRPS model compared to the CRPS model is not very visible. The raw ensemble does not include the forecast control member as this member is not exchangeable with the rest of the ensemble.}
        \label{fig:comp_raw_ensemble}
\end{figure}

\renewcommand{\arraystretch}{1.2}
\begin{table}[h]
\caption{Brier skill score (in \%) over the CRPS-trained model for the twCRPS-trained model for a variety of evaluation thresholds.}
\vspace{5mm}
\centering
\begin{tabular}{cr|r|r|}
\cline{3-4}
\multicolumn{1}{l}{}                       & \multicolumn{1}{l|}{}                    & \multicolumn{1}{l|}{Truncated Normal} & \multicolumn{1}{l|}{Truncated Logistic} \\ \hline
\multicolumn{1}{|l|}{Training Quantile}    & \multicolumn{1}{l|}{Evaluation Quantile} & \multicolumn{1}{l|}{}                 & \multicolumn{1}{l|}{}                   \\ \hline
\multicolumn{1}{|c|}{\multirow{5}{*}{0.8}} & 0.7                                      & -0.62                                 & -0.43                                   \\
\multicolumn{1}{|c|}{}                     & 0.8                                      & -0.26                                 & -0.36                                   \\
\multicolumn{1}{|c|}{}                     & 0.85                                     & 0.14                                  & -0.05                                   \\
\multicolumn{1}{|c|}{}                     & 0.9                                      & 0.52                                  & 0.28                                    \\
\multicolumn{1}{|c|}{}                     & 0.95                                     & 0.52                                  & 0.31                                    \\ \hline
\multicolumn{1}{|c|}{\multirow{5}{*}{0.9}} & 0.7                                      & -3.34                                 & -2.71                                   \\
\multicolumn{1}{|c|}{}                     & 0.8                                      & -1.93                                 & -1.88                                   \\
\multicolumn{1}{|c|}{}                     & 0.85                                     & -0.97                                 & -1.11                                   \\
\multicolumn{1}{|c|}{}                     & 0.9                                      & 0.10                                  & -0.22                                   \\
\multicolumn{1}{|c|}{}                     & 0.95                                     & 0.27                                  & -0.03                                   \\ \hline
\end{tabular}
\label{bss}
\end{table}

\begin{figure}[h]
        \centering
        \includegraphics[width=0.6\textwidth]{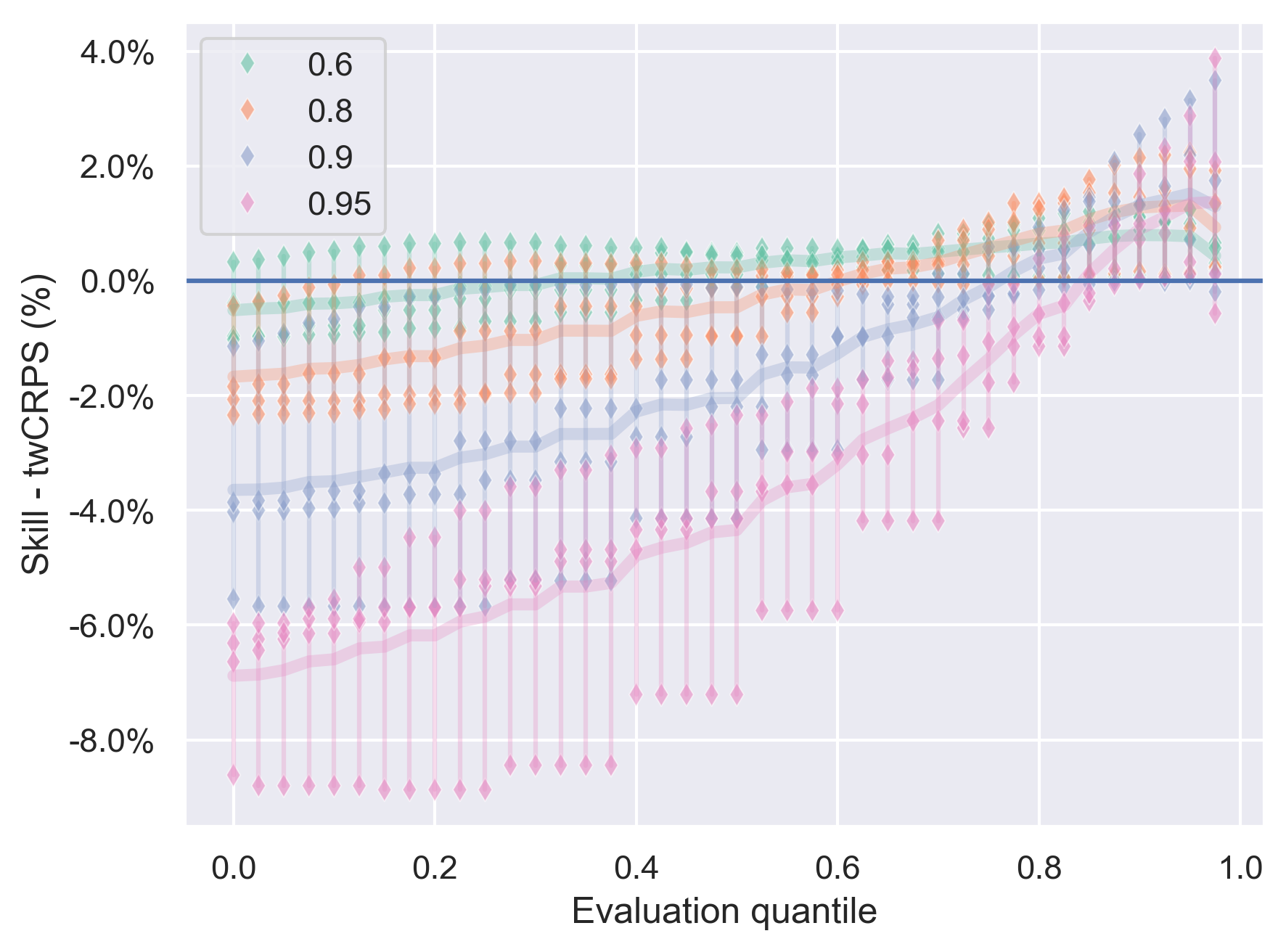}
        \caption{twCRPS skill score (over the CRPS-trained model) as a function of the evaluation threshold for twCRPS-trained models trained using different thresholds.}
        \label{fig:varying_threshold_relative}
\end{figure}

\begin{figure}[h]
        \centering
        \includegraphics[width=0.6\textwidth]{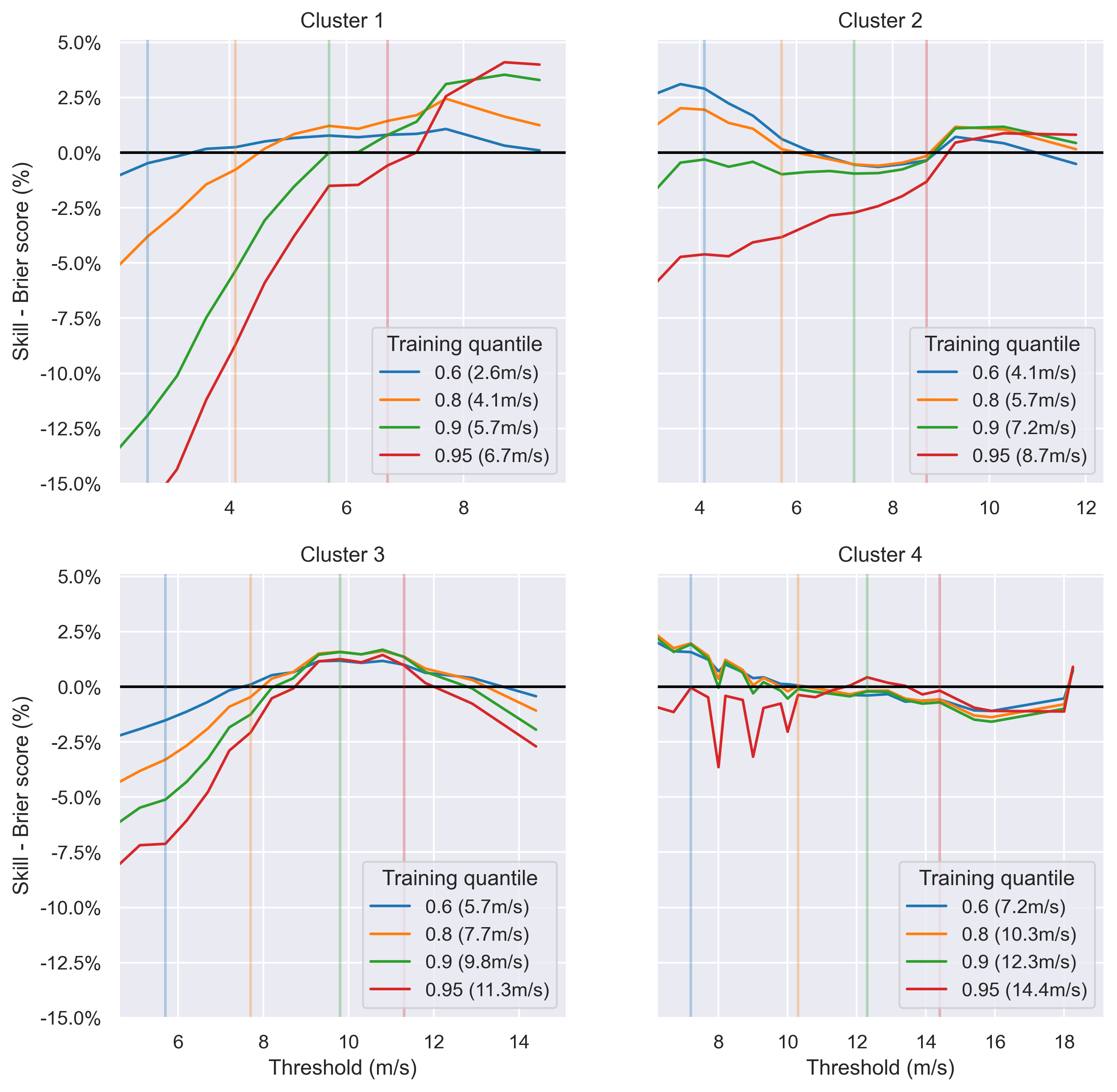}
        \caption{Brier skill score (over the CRPS-trained model) by cluster as a function of the evaluation threshold for twCRPS-trained models trained using different thresholds. Thresholds on the x-axis correspond to wind speeds between the 50th and 99th percentile of the training distribution. The vertical lines indicate the thresholds for which models were trained. Note the different x-axis starting points to Figure \ref{fig:varying_threshold}.}
        \label{fig:brier_score_training_quantiles}
\end{figure}

\end{document}